\documentclass[review,hidelinks,onefignum,onetabnum]{siamart220329}
\usepackage[figuresright]{rotating}
\usepackage{multirow}

\usepackage{lipsum}
\usepackage{amsfonts}
\usepackage{graphicx}
\usepackage{epstopdf}
\usepackage{algorithmic}
\ifpdf
  \DeclareGraphicsExtensions{.eps,.pdf,.png,.jpg}
\else
  \DeclareGraphicsExtensions{.eps}
\fi

\newsiamremark{remark}{Remark}
\newsiamremark{hypothesis}{Hypothesis}
\crefname{hypothesis}{Hypothesis}{Hypotheses}
\newsiamthm{claim}{Claim}

\headers{Adversarial Transferability in Deep Denoising Models}{J. Ning, J. Sun, S. Shi, Z. Guo, Y. Li, H. Li, and B. Wu }

\title{Adversarial Transferability in Deep Denoising Models: Theoretical Insights and Robustness Enhancement via Out-of-Distribution Typical Set Sampling\thanks{Submitted to the editors DATE.
}}

\author{Jie Ning\footnotemark[3]
\and Jiebao Sun\footnotemark[3]
\and Shengzhu Shi\footnotemark[3] \thanks{Corresponding author 
  (\email{mathssz@hit.edu.cn})}
\and Zhichang Guo\footnotemark[3]
\and Yao Li\thanks{School of Mathematics, Harbin Institute of Technology, Harbin, China}
\and Hongwei Li\thanks{School of Mathematical Sciences, Capital Normal University, Beijing, China}
\and Boying Wu\footnotemark[3]
}

\usepackage{amsopn}

\ifpdf
\hypersetup{
  pdftitle={Adversarial Transferability in Deep Denoising Models: Theoretical Insights and Robustness Enhancement via Out-of-Distribution Typical Set Sampling},
  pdfauthor={Jie Ning, Jiebao Sun, Shengzhu Shi, Zhichang Guo, Yao Li, Hongwei Li, and Boying Wu}
}
\fi

\begin{document}

\maketitle

\begin{abstract}
Deep learning-based image denoising models demonstrate remarkable performance, but their lack of robustness analysis remains a significant concern. A major issue is that these models are susceptible to adversarial attacks, where small, carefully crafted perturbations to input data can cause them to fail. Surprisingly, perturbations specifically crafted for one model can easily transfer across various models, including CNNs, Transformers, unfolding models, and plug-and-play models, leading to failures in those models as well. Such high adversarial transferability is not observed in classification models. We analyze the possible underlying reasons behind the high adversarial transferability through a series of hypotheses and validation experiments. By characterizing the manifolds of Gaussian noise and adversarial perturbations using the concept of typical set and the asymptotic equipartition property, we prove that adversarial samples deviate slightly from the typical set of the original input distribution, causing the models to fail. Based on these insights, we propose a novel adversarial defense method: the Out-of-Distribution Typical Set Sampling Training strategy (TS). TS not only significantly enhances the model's robustness but also marginally improves denoising performance compared to the original model.

\end{abstract}



\section{Introduction}
Image denoising is a classical topic in image processing. Mathematically, it can be described as an inverse problem in the following form:
\begin{equation}
	\label{eq:1}
	f(x)=u(x)+n(x),
\end{equation}
where $x\in\Omega$, $f$ is the observed noisy image, $u$ is the clean image, and $n$ is the additive noise. The goal of image denoising is to estimate the clean image while preserving as much detail as possible. Traditional methods include variational-based methods, partial differential equation (PDE)-based methods, non-local methods, and more. However, these traditional approaches primarily rely on handcrafted priors and often struggle to find a prior that generalizes well across different images. The era of deep learning brought a revolution to image denoising, not only leading the way in today’s noise suppression capabilities but also expanding the scope of denoising problems that can be addressed \cite{elad2023image}. Despite this, deep learning-based methods often lack theoretical guarantees for their denoising solutions. Recently, plug-and-play and deep unfolding models have been proposed, combining traditional variational and PDE frameworks with deep learning structures to offer theoretical guarantees for denoising solutions under certain assumptions. However, how the deep learning components impact the underlying mathematical framework remains unclear.

The absence of a solid theoretical foundation raises questions about the robustness of deep learning-based denoising models. For example, the denoised image may exhibit artifacts in regions with uniform color, the denoising performance may decline in real noise removal tasks and blind denoising tasks, and so on. Understanding how and why certain images or types of noise can degrade a well-trained denoising model is critical for researchers to trust and improve deep learning-based methods. Adversarial attacks are a research area focused on identifying methods to cause well-trained deep neural networks to fail. Since they can directly pinpoint the failure cases of deep learning models, they serve as a good starting point for investigating how and why deep learning models fail. Adversarial attacks explore the vulnerabilities of neural networks, aiming to understand their limits and improve their robustness. In 2014, Szegedy et al. \cite{IP} discovered that adding small, carefully chosen perturbations to an input image could cause a deep learning system to make incorrect judgments, potentially leading to serious consequences. The method of generating subtle and undetectable adversarial perturbations is known as an adversarial attack. 

Most existing adversarial attacks \cite{nguyen2015deep, biggio2013evasion} focus on classification tasks, but it is equally destructive in the field of image denoising. Moreover, image denoising, being a regression task with known variable distributions, provides a more suitable context for studying the properties of adversarial attacks. In 2023, Ning et al. \cite{ning2023evaluating} demonstrated that deep image denoising methods exhibit significant vulnerability to adversarial attacks. Introducing imperceptibly small perturbations during the denoising process will lead to large changes in the resulting images. Furthermore, Ning et al. observed that different deep image denoising models exhibit similar adversarial performance. Adversarial perturbations crafted for one model can easily transfer to the other models, including CNNs, unfolding models, plug-and-plays. In contrast, the high adversarial transferability is uncommon in the field of image classification. Image classification models appear to exhibit a greater degree of independence, resulting in adversarial samples that are challenging to transfer effectively across different models. Furthermore, perturbations in image denoising exhibit similar patterns across various samples. These observations indicate that various deep image denoising models behave very much similar at least in the neighborhoods of all the samples. However, the reason was undiscovered in the previous work.

In this paper, we first explore the reason behind the high adversarial transferability in deep image denoising models through a series of hypothesis and validation procedures. Experiments suggest that deep image denoising models effectively learn the characteristics of Gaussian noise. These models exhibit similarities because they all learn the same underlying distribution, specifically the Gaussian distribution. The similarity among models ultimately results in high adversarial transferability. By introducing the concept of the typical set for image denoising, we prove that the Gaussian noises created in the training process are bounded in a small typical set. Adversarial perturbations deviate from the typical set thus easily fail models. However, the deviated noises are still bounded in another relative larger typical set. Based on the above theory and experiments, we propose an adversarial defence method called the Out-of-Distribution Typical Set Sampling Training strategy (TS). TS operates on the premise that adversarial samples are located in low-density regions of the training distribution. Enhancing sampling in these low-density regions can improve model robustness. TS maintains or even improves the denoising performance of the original models compared to traditional adversarial training methods, which improve robustness at the expense of performance.

TS is not only an adversarial defense technique but also a data preprocessing method, and it can increase both the robustness and denoising ability of the model. Specifically, the high adversarial transferability is actually caused by the curse of dimensionality. Since images are high-dimensional data, the noise in such dimensions is limited to a small set by the asymptotic equipartition property or the weak law of large numbers. Thus, even a tiny perturbation can lead to significant out-of-distribution issues. TS expands the training samples to include a diverse array of noise samples, leading to more reliable generalization, even for out-of-distribution samples. The main contributions of our research are summarized as follows:

\begin{itemize}
	\item[$\bullet$] Through experiments, we observe that various deep denoising models, including CNNs, Transformers, unfolding models, and plug-and-plays, exhibit high adversarial transferability among each other. Even different images share the same adversarial pattern. The experimental observation implies that various deep denoising models behave similar in the neighborhoods of all the samples.
    \item[$\bullet$] We explore the reason behind the unexpected adversarial transferability and model similarity from the perspective of the typical set. The high dimensionality limits the training samples of Gaussian noise to a small set, making tiny perturbations sufficient to create out-of-distribution samples that cause the models to fail. For both the original Gaussian noise and perturbed Gaussian noise, we establish theoretical probability bounds and analyze their set volumes compared to the input space.
	\item[$\bullet$] Utilizing the concept of the typical set, we propose an adversarial defense method named TS, which broadens the noise sampling space of training data. TS effectively enhances model robustness and marginally improves denoising performance compared to the original model. Consequently, TS can also be considered as a generalized enhancement technique for high-dimensional denoising.
\end{itemize}

The rest of the paper is organized as follows: we review the most related works in section 2; then the experimental investigation into the causes of adversarial transferability, accompanied by a detailed analysis and demonstration of the properties of the adversarial space in high-dimensional contexts are given in section 3; after that, the algorithm and thoughts for the TS are analyzed in section 4; section 5 presents the experimental results; and finally, we conclude this paper in section 6.

\section{Related Works}

\subsection{Traditional Image Denoising Methods}

Traditional image denoising methods typically formulate this problem as a variational form, consisting of a fidelity term $\mathcal{F}(u;f)$ and a regularization term $\mathcal{R}(u)$ as prior:
\begin{equation}
	\label{eq:2}
	u^*=\underset{u}{\arg\min}\;\mathcal{F}(u;f)+\lambda\mathcal{R}(u),
\end{equation}
where the coefficient $\lambda$ is used to balance the two terms. Among the variational methods, total variation (TV) \cite{TV} regularization is one of the most classical. It assumes the image gradient follows a Laplacian probability distribution, such that solutions belong to the bounded variational space. TV regularization can effectively preserve image edges, however the prior has difficulty differentiating between noise and texture, leading to the “staircase effect”. Higher-order derivatives \cite{rajwade2012image} are later introduced as regularization terms to better recognize texture and details. Unfortunately, higher-order derivatives can cause edge blurry \cite{lysaker2003noise}. On the other hand, variational methods can be solved through gradient flow, leading to the development of many denoising methods directly from the perspective of partial differential equations (PDEs). The most classical PDE-denoiser Perona-Malik (PM) \cite{PM} model introduces a local structures-guided anisotropic diffusion process, allowing edges to be preserved under a slower diffusion rate. It helps to maintain important features in the image while reducing noise. For further developments in variational and PDE-based denoising methods, one can refer to the comprehensive review \cite{salamat2021recent}. Variational and PDE-based methods typically operate under the assumption that the pixels in the solution are primarily influenced by their local neighborhood in the input image. From a non-local perspective, the Block Matching and 3D Collaborative Filtering (BM3D) method \cite{BM3D} is proposed to exploit the self-similarity within the image by searching for similar patches across the entire image. BM3D effectively retains image structures and details, achieving a high signal-to-noise ratio. However, traditional methods primarily rely on manual priors and have difficulty to find the most suitable prior that generalize across different images.

\subsection{Pure Deep Learning-based Image Denoising Methods}

Unlike manually designed priors applied in traditional model-based methods, learning-based methods learn implicit priors from large datasets by optimizing loss functions. Therefore, learning-based image denoising methods achieve significant results that are difficult for traditional methods to match. A milestone in designing deep-learning architectures for image denoising is DnCNN \cite{DNCNN}, which implicitly recovers clean images by predicting noise through a residual learning module. Given that discriminative learning methods need to learn a specific model for each noise level, Zhang et al. proposed the fast and flexible denoising convolutional neural network (FFDNet) \cite{zhang2018ffdnet}. FFDNet includes an adjustable noise level map as an input, thus the model parameters can remain unchanged across different noise levels. FFDNet can be viewed as a collection of multiple denoisers suited to various noise levels. CNN models have limited receptive field and insufficient adaptability to input content. Transformers that have another kind of neural network structure just fill these defects and show significant performance improvement. Liang et al. \cite{liang2021swinir} introduced a model SwinIR based on the Swin Transformer architecture, designed specifically for image restoration tasks. SwinIR includes multiple Swin-Transformer layers and residual connections, which enhance the model's capacity for feature representation. Tian et al. \cite{tian2024cross} embedded the Transformer in serial and parallel blocks to obtain a cross-transformer denoising CNN to extract complementary salient features for improving the adaptability of the denoiser to complex scenes.  

\subsection{Model-data Hybrid-driven Image Denoising Methods}

Considering the uninterpretability of learning-based methods, plug-and-play and unfolding models have been proposed. These models combine traditional variational and PDE frameworks with deep learning structures to provide theoretical guarantees for the denoising solutions under certain assumptions. DPIR \cite{zhang2021plug} is a typical plug-and-play image denoising method. It treats pre-trained CNN denoisers as a module and inserts them into the half-quadratic splitting iterative algorithm to provide prior information. DPIR significantly enhances the effectiveness of model-driven methods due to the powerful implicit prior modeling through deep learning. Without task-specific training, DPIR is more flexible than data-driven methods while offering comparable performance. On the other hand, the main idea of unfolding methods is to unfold the model-based method in the form of neural network structures. Deep unfolding neural networks can achieve high quality image denoising by learning the noise distribution and features of the image. Ren et al. \cite{ren2021adaptive} proposed an innovative adaptive consistency prior (ACP) that integrates a nonlinear filtering operator, a reliability matrix, and a high-dimensional eigen-transform function into the traditional consistency prior framework. By incorporating the ACP term into the maximum a posteriori framework, they developed a new end-to-end trainable and interpretable deep denoising network DeamNet. This integration allows DeamNet to leverage the adaptive consistency prior for more effective noise reduction. Mou et al. \cite{mou2022deep} integrated a gradient estimation strategy into the gradient descent step of the proximal gradient descent algorithm, and proposed DGUNet capable of dealing with complex and real-world image degradation. DGUNet designs inter-stage information pathways across proximal mappings in different projected gradient descent iterations, correcting most deep unfolding networks  through multi-scale and spatial adaption of the inherent information loss.

\subsection{Adversarial Attack}
From the perspective of adversarial attacks, deep neural networks are notably susceptible to adversarial samples. In 2014, Szegedy et al. \cite{IP} first discovered that by adding a carefully chosen tiny perturbation to an input image, the system could be led to make incorrect judgments, potentially causing serious consequences. As shown in Figure \ref{fig:fgsm}, the photo of a panda is classified as standard poodle by the model after adding adversarial noise that is difficult to detect by the human eye.
The algorithms of generating small and imperceptible adversarial perturbations are known as adversarial attacks. Adversarial attacks can be formulated as a optimization problem, where the goal is to maximize the loss function with respect to the perturbations while ensuring the perturbations remain within a specified bound, i.e.,
\begin{equation}\label{eq:adv_attack}
	\max_{\delta\in\mathcal{S}}\mathcal{L}(f(x+\delta;\theta),y)
\end{equation}
where $x$ is the original input and $\delta$ is the perturbation added to the input. $\mathcal{S}$ is the set of allowed perturbations, often bounded by an $l_p$-norm constraint. $\theta$ represents the parameters of the model $f$. $\mathcal{L}$ is a loss function, such as MSE loss for image denoising. $y$ is the true label of the input $x$. The optimization problem seeks to find the perturbation $\delta$ that maximizes the loss $\mathcal{L}$, thereby causing the model $f$ to fail on the perturbed input $x+\delta$. The perturbation is constrained to be within the set $\mathcal{S}$ to ensure it remains imperceptible.

\begin{figure}[tb]
	\centering
	\includegraphics[height=4cm]{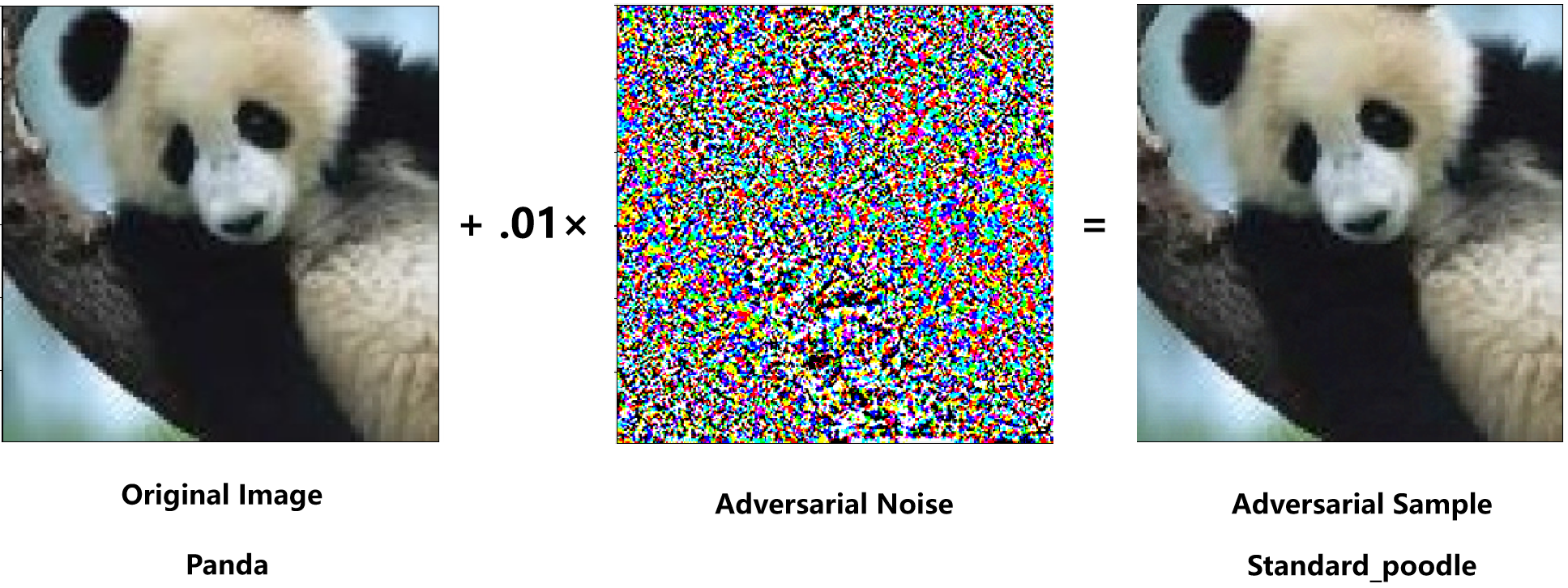}
	\caption{A demonstration of the adversarial attack: by adding a carefully designed, imperceptibly small perturbation to the image input, Inception v3 misclassified the panda as a standard poodle.}
	\label{fig:fgsm}
\end{figure}

From the perspective of manifold learning \cite{chapelle2009semi}, high-dimensional data essentially concentrates near nonlinear low-dimensional manifolds. Adversarial attacks can be viewed as malicious processes that drag benign examples away from these concentrated manifolds. In 2015, Goodfellow et al. \cite{EH} proposed an efficient one-step attack method based on the gradient of a deep neural network, known as the Fast Gradient Sign Method (FGSM). FGSM employs noise mapping and the relationships between image pixels as a loss function. The gradient of the loss function is calculated by adding or subtracting a fixed value based on the sign of each pixel's gradient. Under FGSM, almost all the samples can fool the model with high confidence, despite the changes being imperceptible to the human eye. In 2017, Kurakin \cite{kurakin2016adversarial} showed the Basic Iterative Method (BIM), which extended FGSM into an iterative algorithm by replacing the one-step process with multiple smaller steps. Another iterative variant is the Projected Gradient Descent (PGD) method \cite{mkadry2017towards}, which applies projected gradient descent on a negative loss function to generate adversarial samples and can be described as a saddle point optimization problem. Unlike BIM, PGD introduces random perturbations to the original image before starting the iteration. 

Empirically, adversarial samples can be found near almost any input for most well-trained neural networks. Theoretically, Mahloujifar et al. \cite{mahloujifar2019curse} argue that samples tend to cluster tightly in high-dimensional spaces due to the ``concentration of measure", making it easier for adversarial attacks to succeed. Shafahi et al. \cite{shafahi2018adversarial} propose that, due to the vastness of high-dimensional input spaces, small visual perturbations can significantly deviate samples in certain norms, thus affecting a model's predictions and making adversarial examples unavoidable. Although adversarial samples are common and not difficult to generate, identifying which perturbations are adversarial and developing strategies to avoid them remain challenging. Additionally, the complexity of data distribution and the lack of explicit expressions for the decision boundaries of neural networks further complicate the problem.

\subsection{Defense of Adversarial Attack}

Although the reason behind the existence of adversarial attacks remains an open question, several practical methods for defending against them have been proposed. Defense methods aim to make neural networks robust to adversarial attacks or general small perturbations, with significant practical implications for security.  Existing adversarial defense methods are primarily classified into two categories: robust defenses and detection defenses. Robust defenses aim to improve the adversarial robustness of models by modifying the network structure, training process or preprocessing input data. These modifications enable the model to correctly respond to adversarial samples. In contrast, detection defenses involve setting up an additional detector to accurately identify adversarial samples in the input data, thereby preventing them from being fed into the target model.

Existing literature suggests that detecting and completely eliminating adversarial perturbations is very challenging, as localized residues can still pose significant threats. Additionally, preventing adversarial attacks often requires specific information about the attack techniques, and effectively thwarting all forms of attacks remains challenging. Therefore, adversarial training has become one of the most effective and widely used techniques to improve model robustness. The adversarial training can be represented as a robust optimization problem, i.e.
\begin{equation}
    \theta^* = \arg\min_\theta \mathbb{E}_{(x,y)\sim \mathcal{D}}
	\max_{\delta\in\mathcal{S}}\mathcal{L}(f(\mathbf{x}+\delta;\theta),y)
\end{equation}
where $\mathcal{D}$ denotes the distribution of $(x,y)$. The optimization problem estimates the model by minimizing the empirical risk over the most-adversarial samples. The widespread use of adversarial training has inspired extensive research on improvements and new features. Ford et al. \cite{gilmer2019adversarial} and Yoon et al. \cite{yoon2021adversarial} explored the existence and properties of adversarial spaces from geometric and statistical theoretical perspectives, respectively. Xie et al. \cite{xie2020adversarial} introduced an augmented adversarial training scheme, called AdvProp, which treats adversarial examples as additional examples and uses a separate auxiliary batch specification.

\subsection{Projective Gradient Descent (PGD)}
PGD \cite{mkadry2017towards} is a powerful adversarial attack method that is widely used in the field of image classification, noted for its strong attack performance and flexibility. Mathematically, the PGD attack is a numerical approach of solving the non-convex optimization problem with constrains. For optimization problem in equation (\ref{eq:adv_attack}), PGD iterating the following equation until a stopping condition is met:
\begin{align}
x^{(t+1)} = \mathcal{P}_\mathcal{S}(x^{(t)}-\alpha\cdot\text{sign(}\nabla_x \mathcal{L}(f(x^{(t)};\theta),y))) \notag
\end{align}
where $\mathcal{S} = \{x:\|x-x_0\|_\infty\leq \epsilon\}$ ensures the perturbation is not perceivable to human eyes, $x_0$ is the original image to be attacked, $\alpha\in(0,\infty)$ is the gradient step size, and $\mathcal{P}_\mathcal{Q}(\cdot):\mathbb{R}^n \rightarrow \mathbb{R}^n$ is a projection from $\mathbb{R}^n$ to $\mathcal{Q}$. PGD is a simple algorithm that if the point $x^{(t)}$ after the gradient update is leaving the set $\mathcal{S}$, then project it back; otherwise keep the point. The set $\mathcal{S}$ is constrained by $L_\infty$ norm, thus the projection
\begin{align}
\mathcal{P}_\mathcal{Q} = \text{Clip}_{x_0,\epsilon} (x) \notag
\end{align}
is equavilent to clip every dimensions of $x$ within $[-\epsilon,\epsilon]$ neighborhood of $x_0$. The sign function ensures that the perturbation affects each dimension of the input equally in terms of magnitude, regardless of the scale of the gradient in that dimension. This perturbation approach can be more effective in high-dimensional spaces, like images. The PGD attack method also induces a random initialization that adds a uniformly distributed random variable $\mathcal{U}[-\epsilon,\epsilon]$ to $x_0$ before the first iteration. Random initialization helps to prevent the formation of local minima on the training samples during model training. The authors of PGD verified in \cite{mkadry2017towards} that if a network is trained to be robust to PGD adversarial samples, it is also robust to a wide range of other attacks.

\subsection{Deep Image Denoising Adversarial Attacks}
Deep image denoising models are typically regarded as uninterpretable black-box systems, leading to limited discussion on their security, robustness, and noise generalization capabilities. Ning et al. \cite{ning2023evaluating} discussed the robustness of these models from an adversarial perspective. They introduced a PGD-based adversarial attack method, Denoising-PGD, which successfully degraded the performance of various denoising models by adding small perturbations to the input images. Denoising-PGD introduced extra artifacts and textures to the output of CNN, unfolding, and plug-and-play models while maintaining the distribution of the input noise almost unchanged. As illustrated in the last subplot of Figure \ref{f:advshow}, the denoising result of the perturbed image contains textures in flat areas, which significantly degrade the denoising performance of DnCNN.

\begin{figure}[!t]
	\centering
	\includegraphics[height=6cm]{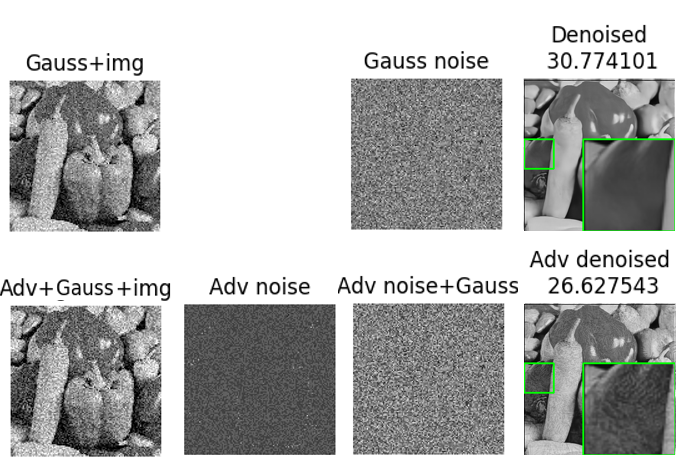}
	\caption{The performance of DnCNN decreases under the Denoising-PGD adversarial attack. First row: DnCNN successfully denoised the input image. Second row: DnCNN suffers a significant performance decrease when applied to the adversarial sample.}
	\label{f:advshow}
\end{figure}

Furthermore, Ning et al. observed that all the adversarial samples specifically generated for the DnCNN model can also cause failures in other models. This phenomenon, where adversarial samples can transfer among various models, is known as adversarial transferability, which can occur with some adversarial samples. For well trained deep image denoising models $f_1$ and $f_2$, an adversarial sample designed to attack $f_1$ is said to possess adversarial transferability if it induces a similar adversarial effect on $f_2$. The mathematical expression for the adversarial transferability is:
\begin{align}
	\left\{\begin{array}{c|c}
		& \mathcal{L}\left(f_{1}(x), y\right)<\mathcal{L}(x,y) \\
		x^\prime = adv(x,f_1) & \mathcal{L}\left(f_{2}(x), y\right)<\mathcal{L}(x,y) \\
		&\mathcal{L}\left(f_{1}(x^{\prime}), y\right)-\mathcal{L}\left(f_{1}\left(x\right), y\right)>M \\
        &\mathcal{L}\left(f_{2}(x^{\prime}), y\right)-\mathcal{L}\left(f_{2}\left(x\right), y\right)>M
	\end{array}\right\}
\end{align}
where the adversarial sample ${x}'=adv(x,f_1)$ is generated specifically for the model $f_1$ by the adversarial attack $adv(\cdot )$ on the noisy image $x$. We say that the adversarial sample is transferable when $x^\prime$ can successfully attack both $f_1$ and $f_2$, where $\mathcal{L}$ is a loss and $M$ is a threshold indicating the success of the attack.

Surprisingly, Denoising-PGD demonstrated strong adversarial transferability, i.e., all the adversarial samples could transfer to all the tested models. As shown in Figure \ref{f:alladvper}, the adversarial sample was designed for DnCNN but could also attack other models, including FFDNet\cite{zhang2018ffdnet}, DnCNN-B\cite{DNCNN}, DPIR\cite{zhang2021plug}, ECNDNet\cite{tian2024cross}, RDDCNN-B\cite{zhang2023robust}, and DeamNet\cite{ren2021adaptive}. This phenomenon was observed across all the tested input images, suggesting that these models share a nearly identical set of adversarial samples. Moreover, the tested models include non-blind denoising models, blind denoising models, plug-and-play models, and unfolding denoising models. Therefore, the observation that the models, regardless of their structure or whether they are model-data hybrid-driven, share a nearly identical set of adversarial samples is astonishing.

\begin{figure}[!t]
	\centering
	\includegraphics[height=10cm]{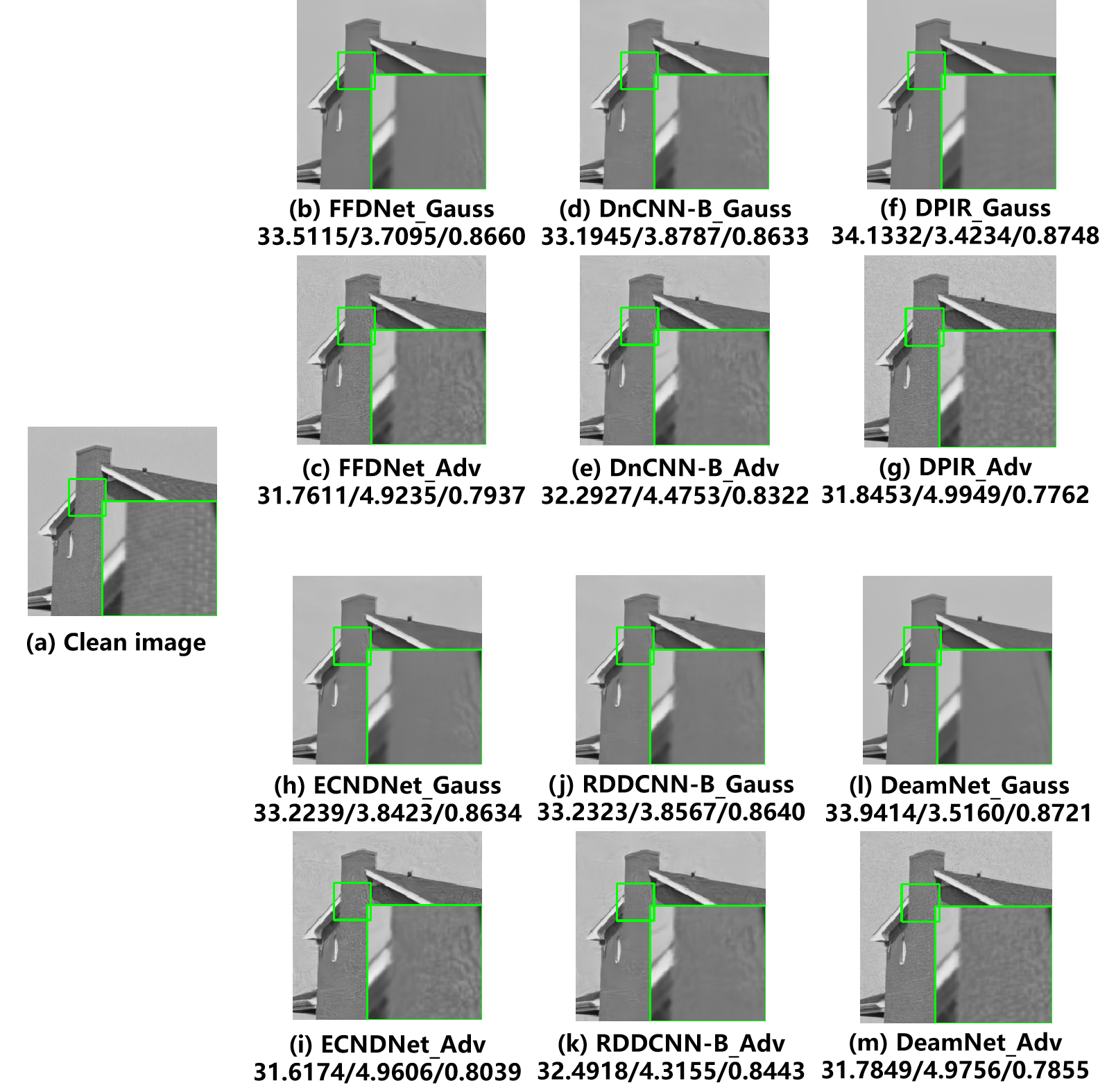}
	\caption{Model performance on the adversarial sample generated for DnCNN. Subplots (b), (d), (f), (h), (j) , and (l) are model outputs for original noisy images. Subplots (c), (e), (g), (i), (k), and (m) are model outputs for adversarial samples. The metrics are PSNR, MAE, and SSIM.}
	\label{f:alladvper}
\end{figure}

The unexpectedly strong transferability may be attributed to either the adversarial attack method or the characteristics of the model itself, as these are the only elements involved in the adversarial attack. In more detail, it could be due to two potential reasons:
\begin{itemize}
\item \textit{Hypothesis 1.1}: (Attack is strong.) The adversarial attack method itself has strong transferability, allowing it to find adversarial samples that are transferable across all the models.
\item \textit{Hypothesis 1.2}: (Models are similar.) The models behave similar in the neighborhoods of all the samples, making any adversarial sample transferable to other models.
\end{itemize}
However, the commonly accepted fact is that the PGD attack has poor transferability. The key idea of the PGD attack method is to focus on the local gradient of the sample, which is why transferability is not its strength. Its poor transferability in image classification tasks is consistent with its algorithm characteristic. It implies that the strong transferability between deep image denoising models is not brought about by adversarial attack methods, but is the property of denoising models. In other words, the strong adversarial transferability suggests that deep image denoising models exhibit similar behaviors in the neighborhoods of all input samples.  From the adversarial perspective, it implies that the denoising models are very much similar. However, the design and training motivations behind these deep models are quite disparate. The observed similarity in their adversarial behavior is therefore counterintuitive, prompting further investigation to uncover the underlying reasons. In this paper, we begin with a series of hypotheses and experiments to provide explanatory evidence for the transferability of adversarial samples.

\section{Investigation and Explanation}

In this section, we first investigate the reasons behind the high adversarial transferability in image denoising through a series of hypotheses and experiments. Guided by the experimental observations, we then propose a theoretical explanation based on the concept of the typical set. Then, a bound is derived for all adversarial samples. Building on proposed theory, we introduce an adversarial defense method called Out-of-Distribution Typical Set Sampling (TS). TS not only enhances the model's robustness but also improves performance, making it suitable as a data augmentation technique. Importantly, TS is a sampling strategy that does not alter the model's structural design, allowing it to be easily applied to a wide variety of models to further enhance both robustness and denoising performance.

\subsection{Model Similarity in the Neighborhoods of All the Sample}
In the previous section, based on the poor transferability of the PGD attack, we suggest that \textit{Hypothesis 1.2} (models are similar), instead of \textit{Hypothesis 1.1} (attack is strong), is true. In this subsection, we visualize the neighborhoods of all samples across the models to verify \textit{Hypothesis 1.2}.

We treat images, noise and perturbations as high-dimensional vectors. Then, we set the noisy image $u+n$ as the origin, use the Gaussian noise $n$ and the adversarial perturbation $v$, which is generated through denoising-PGD, to construct a two-dimensional subspace. The subspace can be visualized as shown in Figure \ref{2dsphere}. Specifically, the direction of the Gaussian noise is set at the $0$ degree, which corresponds to the positive vertical direction. In this subspace, we sample perturbations by rotating around the origin with various radii. Denoting
\begin{align}
    e_1 &= \frac {n}{\|n\|_2} \notag \\
    e_2 &= \frac{v-<v,n>n}{\|v-<v,n>n\|_2}\notag
\end{align}
as two orthonormal basis constructed from $n$ and $v$ through the Gram-Schmidt orthogonalization process, the sampling process is defined as:
\begin{align}
	\label{eq:10}
    s&=\left(e_1\cos\theta+e_2\sin\theta\right)\gamma\|v\|_2+n+u
\end{align}
where $\theta=2\pi i/N_i$,  $\gamma=j/N_j$, $i=1,\hdots,N_i$, $j=1,\hdots, N_j$. Based on Equation \ref{eq:10}, the samples belong to the subspace constructed by the basis vectios $n$ and $v$. They rotate from $0$ to $2\pi$ around the origin $n+u$, with the radius varying from $0$ to $\|v\|$. The samples are perturbations in the neighborhoods of $u+n$ within a selected subspace of interest. 

\begin{figure}[tb]\label{2dsphere}
	\centering
	\includegraphics[height=8cm]{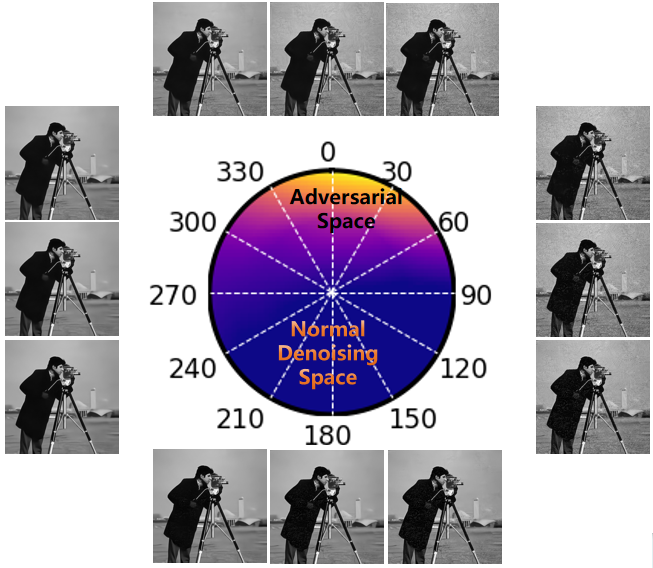}
	\caption{Radar diagram of samples in the neighborhood of the "Cameraman" image. Brighter colors indicate a stronger decrease in denoising performance.}
\end{figure}

The model performance of DnCNN on samples $s$ are shown in Figure \ref{2dsphere}. DnCNN experiences a significant performance decrease on samples generated around $0$ degree, but shows no performance decrease on samples generated from $90$ degrees to $270$ degrees. This indicates that the model is vulnerable only to perturbations around $0$ degree. For clarity, we reorganize the observations as follows:
\begin{itemize}
    \item \textit{Observation 1}: (connected set) Adversarial perturbations appear to belong to a connected set rather than multiple distinct regions.
    \item \textit{Observation 2}: (parallel to Gaussian) The model is vulnerable to perturbations that have directions similar to the Gaussian noise added to the image.
\end{itemize}

We further test the samples generated for DnCNN on various models. \textit{Observation 1} (connected set) and \textit{Observation 2} (parallel to Gaussian) are consistent across all the images in Set12, various of Gaussian noise, and all the tested models, including DnCNN\cite{DNCNN}, ECNDNet\cite{tian2024cross}, DnCNN-B\cite{DNCNN}, RDDCNN-B\cite{zhang2023robust}, DeamNet\cite{ren2021adaptive}, and DPIR\cite{zhang2021plug}. The result is shown in Figure \ref{f:alladvrange}. In the neighborhoods of all the tested images, all the models exhibit similar adversarial attack regions, allowing any adversarial perturbation transferable between models. Thus, this experiment supports \textit{Hypothesis 1.2}, i.e. the strong transferability is caused by denoising models instead of the attack algorithm.

\begin{figure}[!t]
	\centering
	\includegraphics[width=.8\textwidth]{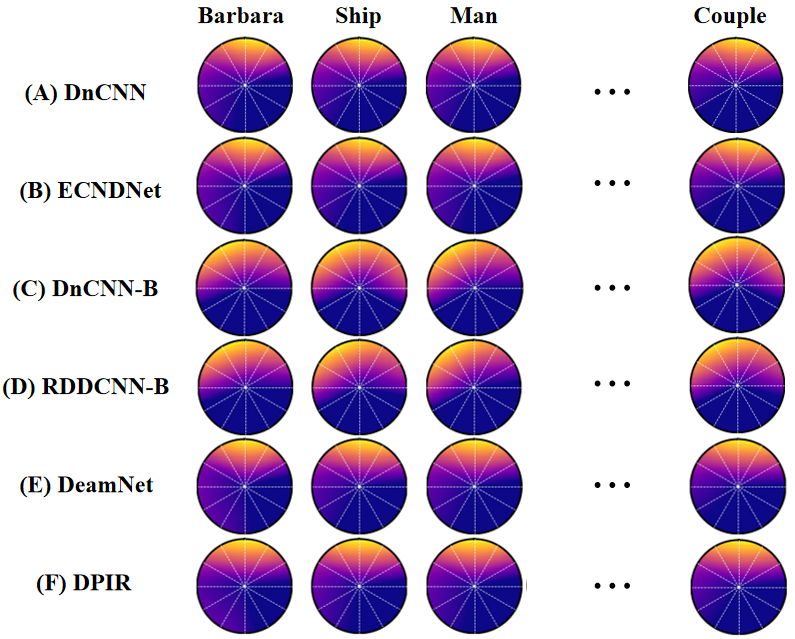}
	\caption{Adversarial attack regions of various models under adversarial samples generated for DnCNN.}
	\label{f:alladvrange}
\end{figure}

Another interesting observation is that
\begin{itemize}
    \item \textit{Observation 3}: (content free) Deep image denoising models demonstrate the similar adversarial space on images with different semantic contents.
\end{itemize}
It implies that the model’s local performance appears to be related only to the Gaussian noise, rather than the image content. The objective of the image denoising model is to recover clean images from noisy inputs while remaining unaffected by the inherent information contained in the input images. Thus, this characteristic is considered an ideal property for effective denoising models.

\subsection{Adversarial Perturbations Belong to a Single Continuous Set} 

In the previous section, we observed that adversarial perturbations appear to belong to a single continuous region within the 2-dimensional subspace constructed by $n$ and $v$. In this experiment, we will extend the analysis to a 3-dimensional subspace and verify \textit{Observation 1} (connected set) with different Gaussian noises added to the same image.

\begin{figure}[tb]\label{twosphere}
	\centering
	\includegraphics[width=.8\textwidth]{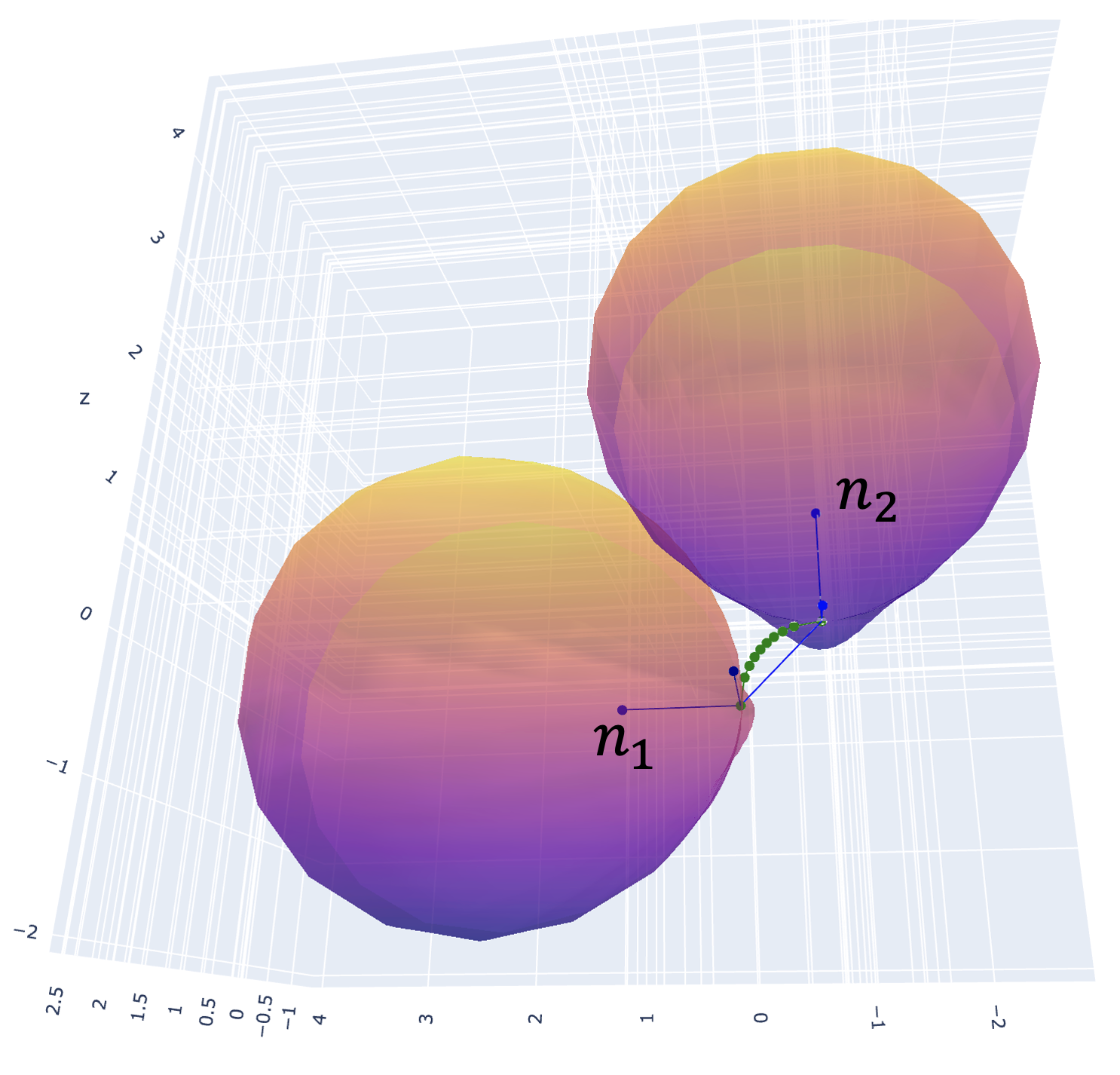}
	\caption{Schematic representation of the 3D confrontation space in the case of two Gaussian noise sampling. Where the x-axis y-axis z-axis are the Gaussian noise n, the Gaussian noise v and the Gaussian noise $n_1-n_2$ respectively. The green curve shows the convex combination path crafted for the two adversarial noises.}
	\label{fig:3}
\end{figure}

For the clean image $u$ and two different Gaussian noise $n_1$ and $n_2$, we create a 3-dimensional subspace at the origin $u$. The basis vectors are $n_1$, $n_2$, and $v_1$, where $v_1$ is the adversarial perturbation generated for $u+n_1$. In this experiment, we apply a different visualization for adversarial attack regions. As shown in Figure \ref{fig:3}, the origins of the two coordinates are $u+n_1$ and $u+n_2$. The directions of the vectors plotted are $n_1$, $n_2-n_1$, and $v_1$ for the first coordinate, and $n_2$, $n_1-n_2$, $v_2$ for the second coordinate. For each Gaussian noise $n_k$ (where $k=1,2$), denoting 
\begin{align}
    e_{1} &= \frac{n_1}{\|n_1\|_2} \notag \\
    e_{2} &= \frac{n_2-<n_2,e_{1}>e_{1}}{\|n_2-<n_2,e_{1}>e_{1}\|_2} \notag \\
    e_{3} &= \frac{v_1-<v_1,e_{1}>e_{1}-<v_1,e_{2}>e_{2}}{\|v_1-<v_1,e_{1}>e_{1}-<v_1,e_{2}>e_{2}\|_2} \notag 
\end{align}
as orthonormal basis constructed from $n_1$, $n_2$ and $v_1$ through the Gram-Schmidt orthogonalization process, we generate adversarial samples on the sphere centered at $u+n_k$ with a radius of $\|v_k\|_2$ as:
\begin{align}
    s_k&=\left(e_1\cos\phi\cos\theta+e_2\cos\phi\sin\theta+e_3\sin\phi\right)\|v_k\|_2+n_k+u
\end{align}
where $\theta=2\pi i/N_i$, $\phi=\pi j/N_j - \pi/2$, $i=1,\hdots,N_i$, $j=1,\hdots,N_j$. In Figure \ref{fig:3}, we plot
\begin{align}
\left(e_1\cos\phi\cos\theta+e_2\cos\phi\sin\theta+e_3\sin\phi\right)a_k+n_k+u \notag 
\end{align}
to present the adversarial region, where $a_k$ is the model's performance decrease under adversarial samples, such as PSNR values. As demonstrated, the adversarial samples parallel to the Gaussian noises $n_1$ and $n_2$ have the most significant adversarial effect, which is consistent with \textit{Observation 2} (parallel to Gaussian). The adversarial regions intersect with each other, supporting \textit{Observation 1} (connected set).

Further, we exam the connection between the adversarial regions. First, 
we generate more Gaussian noises combining $n_1$ and $n_2$ as
\begin{align}
n_k = \sqrt{\lambda}n_1+\sqrt{1-\lambda}n_2\notag
\end{align}
where $\lambda=(k-2)/N_k$, $k=3,\hdots,N_k+1$. It can be easily verified that $n_k$ are Gaussian white noises with the variance maintained. The Gaussian noises are presented in Figure \ref{fig:3} using green dots. It can be seen that the green dots are located beyond the adversarial regions. Then, we generate more adversarial samples combining
$v_1$ and $v_2$ as
\begin{align}
v_k = \sqrt{\lambda}(v_1+n_1)+\sqrt{1-\lambda}(v_2+n_2)\notag
\end{align}
where $\lambda=(k-2)/N_k$, $k=3,\hdots,N_k+1$. Experimentally, all the adversarial sample $v_k$ show adversarial effect. It implies that the adversarial regions corresponding to two arbitrary different Gaussian noises are connected. More generally speaking, the adversarial regions of an image $u$ are a single connected set. Thus, the experiment supports \textit{Observation 1} (connected set).

\subsection{Hypothesis for the Cause of Similarity}

As demonstrated above, the experiments support \textit{Hypothesis 1.2}, i.e. the models are similar in the neighborhoods of all the samples, making any adversarial sample transferable to other models. However, the designs and motivations behind the tested models are disparate. For example, DnCNN is a non-blind denoising model, RDDCNN-B is a blind denoising model, DPIR is a model-data hybrid-driven plug-and-play model, and DeamNet is a model-data hybrid-driven unfolding model. Intuitively, these models should be quite different. Therefore, under \textit{Hypothesis 1.2} (models are similar), we need to determine the reason for the model similarity across all the samples. This similarity could be due to the following reasons:
\begin{itemize}
    \item \textit{Hypothesis 2.1}: The model structures are similar.
    \item \textit{Hypothesis 2.2}: The image contents in the training datasets are similar.
    \item \textit{Hypothesis 2.3}: The artificial Gaussian noise used in the training process caused the similarity.
\end{itemize}
To the best of our knowledge, the three hypotheses mentioned above are the only possibilities we believe could result the similarity, as the model structure, training dataset, and artificial Gaussian noise are the primary elements determining deep image denoising models. In the following experiments, we will exclude \textit{Hypothesis 2.1} and \textit{Hypothesis 2.2}.

\subsubsection{Model Structural Similarity}

We hypothesise that the strong adversarial transferability of adversarial samples in deep image denoising models is a result of the high structural similarity between these models. DnCNN is a foundational model in the field of deep image denoising, which has inspired a number of deep image denoising models to design similar structural designs. The deep image denoising models proposed in recent years such as PANet\cite{mei2023pyramid} and RDDCNN are structurally similar to the DnCNN architecture. This phenomenon has led to a relatively high structural similarity between deep image denoising models. And in the adversarial studies of image classification, it has been observed that models with high network structural similarity are more likely to exhibit adversarial transferability between them. Specifically, image classification models based on the ResNet architecture have shown better adversarial transferability. However, adversarial samples produced for ResNet models showed significantly lower adversarial performance on models with architectures such as VGG.

In order to analyse the relationship between model structural similarity and adversarial transferability, we test the adversarial samples generated for the DnCNN model on SwinIR\cite{liang2021swinir}, which is based on the Transformer architecture.  Transformer is an effective model architecture, which is different from the CNN architecture. While CNNs are highly effective for image-related tasks due to their ability to capture spatial hierarchies, Transformers excel in handling sequential data and capturing long-range dependencies. As demonstrated in Table \ref{ta:transadv}, the Transformer model exhibits a significant adversarial effect on the adversarial samples generated for DnCNN. The image quality evaluation metrics show a heavy decrease in PSNR and SSIM. The experimental results show strong transferability between the CNN and Transformer structures in the deep image denoising problem. Thus, the model similarity is not caused by the structure similarity and \textit{Hypothesis 2.1}  shall be rejected.

\begin{table}
	\caption{The performance of SwinIR on adversarial samples generated for DnCNN.}
	\fontsize{7}{9}\selectfont
	\centering
	\begin{tabular}{ p{1.5cm}p{1cm}p{1.6cm}p{1cm}p{1.5cm}p{1cm}p{2cm} } 
		\hline
		\hline
		& Gauss & Adv & Gauss & Adv & Gauss & Adv \\
		Image & PSNR & PSNR(RES) & MAE & MAE(RES) & SSIM & SSIM(RES) \\
		\hline
		Cameraman & 30.2 & 29.4(-0.8) & 5.2 & 6.1(+0.9) & 0.87 & 0.80(-0.07) \\
		House & 34.2 &31.9(-2.3) & 3.4 & 4.9(+1.5) & 0.88 & 0.79(-0.09) \\
		Pepper & 31.2 &30.0(-1.2) & 4.9 & 6.0(+1.1) & 0.89 & 0.81(-0.08) \\
		Fishstar & 29.9 &29.3(-0.6) & 5.9 & 6.4(+0.5) & 0.88 & 0.86(-0.02) \\
		Monarch & 30.8 &29.8(-1.0) & 4.9 & 6.0(+1.1) & 0.93 & 0.86(-0.07) \\
		Airplane & 29.4 &28.8(-0.6) & 5.5 & 6.2(+0.7) & 0.88 & 0.83(-0.05) \\
		Parrot & 29.3 &28.8(-0.5) & 6.0 & 6.7(+0.7) & 0.86 & 0.81(-0.05) \\
		Lena & 31.4 &30.2(-1.2) & 4.6 & 5.8(+1.2)& 0.90 & 0.83(-0.07) \\
		Barbara & 30.3 &29.5(-0.8) & 5.6 & 6.3(+0.7)& 0.89 & 0.86(-0.03) \\
		Ship & 29.3 &28.4(-0.9) & 6.1 & 7.1(+1.0)& 0.84 & 0.79(-0.05) \\
		Man & 29.0 &28.4(-0.6) & 6.5 & 7.2(+0.7)& 0.82 & 0.78(-0.04) \\
		Couple& 29.1 &28.4(-0.7) & 6.4 & 7.2(+0.8)& 0.85 & 0.81(-0.04) \\
		\hline
		Avg. & 30.4 &29.4(-1.0) & 5.4 & 6.3(+0.9)& 0.87 & 0.82(-0.05) \\
		\hline
		\hline
	\end{tabular}
	\label{ta:transadv}
\end{table}

\subsubsection{Training Dataset Similarity}

Deep neural networks are mainly composed of two parts: the network structure and the training data. After excluding the effect of network structure similarity on adversarial transferability, we hypothesise that the high similarity of the training data leads to high adversarial transferability between deep image denoising models. The datasets used by deep image denoising models are often highly similar, such as DIV2K, BSD500 etc. The number of images used to training a deep denoising model is not as large as the datasets used for image classification, which contributes to the high degree of similarity in the training data in image denoising. Furthermore, under adversarial transferability analysis in image classification, it is often assumed that two deep network models are trained on the same dataset. This setting may help to improve the adversarial transferability between models.

If the adversarial transferability of a model is related to the training dataset similarity, then  irrelevant the training data should weaken the adversarial transferability of the model. To test this hypothesis, we trained two models DnCNN-1 and DnCNN-2, using entirely different training sets.  Specifically, we use two semantically-independent datasets for training DnCNN models, the animal dataset and the landscape dataset. Subsequently, we generated adversarial samples on DnCNN-1 and DnCNN-2 and evaluated their adversarial transferability on the other model. As shown in Figure \ref{fig:datachange}, the adversarial samples generated by DnCNN-2 exhibit the adversarial performance on DnCNN-1. DnCNN-2 exhibits the same phenomenon with the same setup. This demonstrates that adversarial samples exhibit strong transferability between deep image denoising models, even trained on different datasets. Thus, the model similarity is not caused by the dataset similarity and \textit{Hypothesis 2.2} can be rejected.

\begin{figure}[tb]
	\centering
	\includegraphics[height=9cm]{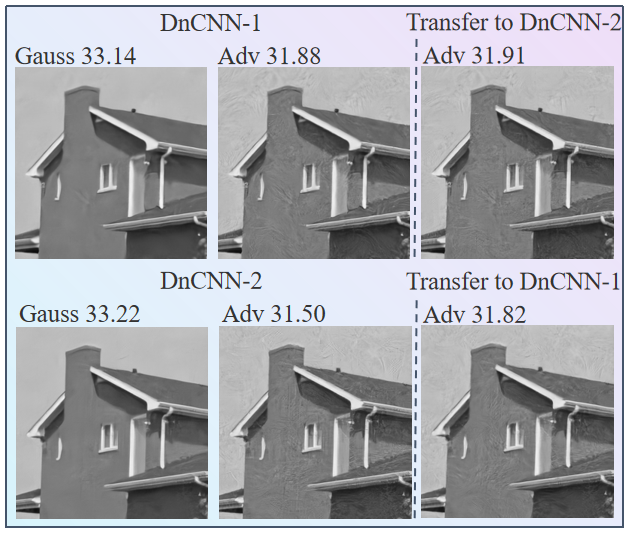}
	\caption{DnCNNs trained with different datasets have high adversarial transferability.}
	\label{fig:datachange}
\end{figure}

\subsubsection{Artificial Gaussian Noise in Patches}

After excluding \textit{Hypothesis} 2.1 and \textit{Hypothesis} 2.2, the high transferability is most likely caused by the artificially added Gaussian noise. The key difference between the image denoising task and the image classification task is the process of adding noise during the training of denoising models. Additionally, \textit{Observation 3} (content free) suggests that performance is related only to the Gaussian noise, not the image content. In the training process of deep image denoising models, images in the training dataset are segmented into small patches, with i.i.d. multivariate Gaussian noise added to each patch. Therefore, if certain properties of the artificial Gaussian noise dominate the adversarial attack, the attack should also be effective on each small patch. In the following experiment, we will examine whether a small patch of the adversarial perturbation can successfully attack the model.

Goodfellow et al. \cite{IP} argued that adversarial attacks can be viewed as the accumulation of the effects of perturbations on the model at each pixel. The high dimensionality of the data implies that small perturbations per pixel, in the limit of the $L_p$ norm, can have a significant impact on the output. If this statement holds true for deep image denoising, the adversarial perturbation will likely only be effective when applied to the model as a whole. Using a small patch of the perturbation will fail to attack the model. However, if the adversarial effect is primarily caused by the i.i.d. Gaussian noise, even a small patch of the perturbation could lead to a localized decrease in performance.

We generate the adversarial perturbation patches using two different methods. In the first method, we apply Denoising-PGD to generate an adversarial perturbation patch specifically for the image patch, then replace the original Gaussian noise in the patch with this adversarial perturbation. In the second method, we apply Denoising-PGD to generate an adversarial perturbation for the entire image, but retain the perturbation only within the region of the patch. As shown in Figure \ref{patch}, the $50\times 50$ patch is generated using the first method and the $30\times 30$ patch is generated using the second method. The adversarial perturbations generated by both methods can attack the model locally.

\begin{figure}[tb]\label{patch}
	\centering
	\includegraphics[height=3.5cm]{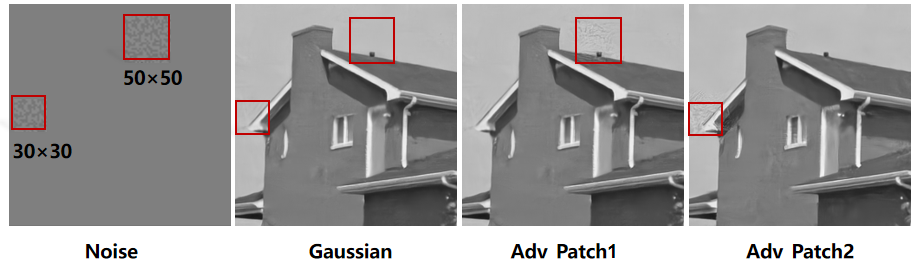}
	\caption{Effectiveness of adversarial attacks on the 50*50 and 30*30 regions of the image (\emph{Left}) and comparison with Gaussian noise denoising (\emph{Right})}
	
\end{figure}

Based on the experiment, adversarial attacks are localized in the image denoising task. This contrasts with the image classification task, where localized perturbations do not necessarily alter the classification result. Thus, it allowing the consideration of adversarial effects at the level of individual training patches without necessitating an analysis of the adversarial noise across the entire image. Additionally, the experiment supports \textit{Hypothesis 2.3}, i.e. the artificial Gaussian noise used in the training process caused the similarity.

\subsection{Theoretical Explanation for the Cause of Similarity}

As discussed in the hypotheses and experiments in the previous section, the high adversarial transferability is due to the similarity of models in the neighborhoods of all the samples. This similarity is most likely caused by the i.i.d. artificial Gaussian noise added during the training process. 

In this section, we propose a theoretical explanation for the cause of the similarity from the perspective of typical set. The i.i.d. multivariate Gaussian noise added to the clean image can be treated as a sequence of i.i.d. Gaussian random variables. According to the law of large numbers, the sample mean converges to the expectation of the Gaussian random variable. Similarly, when the sequence is considered as a multivariate random variable, it belongs to a typical set in probability. Consequently, we can use this set to describe all instances of artificial Gaussian noise added to images during the training process.

As will be demonstrated in the following theories, tiny perturbations can alter the probability densities of the samples, thus causing deviations from the typical set and leading to model failure. For a fixed $L_\infty$ or $L_2$ norm, when the direction of the perturbation aligns with the Gaussian noise, the adversarial sample deviates most from the typical set. This alignment causes all the samples to have similar adversarial attack regions. Thus, models have high adversarial transferability. In the following we will present some preliminaries:

\subsubsection{Preliminaries}

Entropy is a measure about uncertainty or randomness of random variables. It is defined as follows:
\begin{equation}
	H(X)=-\sum_{x\in\mathcal{X}} P( x) \log P(x)
\end{equation}
where $X$ is a discrete random variable, $\mathcal{X}$ is the sample space, $P(x)$ is the probability mass function, and $\log$ represents the $\log_2$ function. 
If $X$ is a continuous random variable, it is assumed to have an infinite number of possible values, resulting in theoretically infinite uncertainty. Consequently, its entropy would be infinite. Therefore, differential entropy is used instead and is defined as follows:
\begin{equation}
	h(X)=-\int_{x\in S}f(x)\log f(x)\text{d}x
\end{equation}
where $S$ is the support set of $f(x)$.

The weak law of large numbers (WLLN) states that as the number of samples increases, the sample mean converges to the expected value of the random variable in probability. 

\begin{theorem} \label{thm:wlln}(WLLN) \cite{thomas2006elements}
	Let $X_1, X_2, \hdots, X_n$ be a sequence of i.i.d. random variables, then
	\begin{align}
		\frac{1}{n}\sum_{i=1}^n X_i \rightarrow \mathbb{E}[X]  \text{ in probability}\notag
	\end{align}
\end{theorem}
Based on the weak law of large numbers, an important property in information theory called the asymptotic equipartition property (AEP) can be derived.

\begin{theorem}\label{thm:aep} (AEP) \cite{thomas2006elements} Let $X^{(n)} = X_1, X_2, \hdots, X_n$ be a sequence of continuous random variables drawn i.i.d. according to the density $f(x)$. Then
	\begin{align}
		-\frac{1}{n}\log f(X^{(n)}) \rightarrow \mathbb{E}[-\log f(X)] = h(X) \text{ in probability}\notag
	\end{align}
	where $h(X)$ is the differential entropy of $X$. 
\end{theorem}

AEP is a similar property as WLNN in that they both describe the deterministic aspects of an i.i.d. random variables sequence, although WLLN focuses on the expectation while AEP focuses on the probability. In fact, AEP can be proved directly by WLLN. Based on AEP, the typical set is defined as:

\begin{definition}\label{def:ts} (Typical Set) \cite{thomas2006elements}
	The typical set of the i.i.d. random variable sequence $X^{(n)}$ drawn according to the density $f(x)$ is defined as
	\begin{align}
		A_\epsilon^{(n)}(X) = \left\{ (x_1,x_2,\hdots,x_n) : \Big|-\frac{1}{n} \log f(x_1,x_2,\hdots,x_n)-h(X)\Big| \leq \epsilon \right\} \notag
	\end{align}
	where $f(x_1,x_2,\hdots,x_n)=\prod_{i=1}^n f(x_i)$.
\end{definition}

\begin{theorem} \cite{thomas2006elements} \label{thm:ts} Typical set has the following properties
	\begin{enumerate}
		\item $P\{A_\epsilon^{(n)}(X)\}:= P\{X^{(n)}\in A_\epsilon^{(n)}(X)\} \xrightarrow{n\rightarrow \infty} 1$ with probability
		\item $\forall n$, $\text{Vol}(A_\epsilon^{(n)}(X)) \leq 2^{n(h(X)+\epsilon)}$
		\item $\forall n$ sufficient large, $\text{Vol}(A_\epsilon^{(n)}(X)) \geq (1-\epsilon) 2^{n(h(X)-\epsilon)}$
	\end{enumerate}
	where $\text{Vol}(A)=\int_A \text{d}x_1\text{d}x_2\hdots \text{d}x_n$.
\end{theorem}

According to these properties, as the length of the sequence $n$ increases, the probability that the sequence belongs to the typical set tends to 1. However, the volume of the typical set remains bounded and small. Based on the definition of typical set and its properties, we describe the typical set of the Gaussian noise added to the image in the following section. We will show that the typical set of the Gaussian noise lies on a sphere when the dimension $n$ tends to infinity.

\subsubsection{Typical Set of the Artificial Gaussian Noise} The artificial Gaussian noise sampled during the training process can be viewed as an i.i.d. Gaussian random variable sequence. Denote the multivariate Gaussian noise added to an image patch as $X^{(n)}\sim\mathcal{N}(\mathbf{0},\sigma^2\mathbf{I}_n)$, where $X\sim \mathcal{N}(0,\sigma^2)$, $n=H\times W$, $H$ and $W$ are the height and width of the image patch, and $\sigma$ is the standard deviation. By Theorem \ref{thm:ts}(1), the probability $P\{X^{(n)}\in A_\epsilon^{(n)}(X)\} \xrightarrow{n\rightarrow \infty} 1$. Since $n$ is large for a image patch, $X^{(n)}$ has a high probability of belonging to the typical set $A_\epsilon^{(n)}(X)$. Thus, during the training process of the deep denoising model, the Gaussian noise we generate is actually sampled from the typical set with a probability approximately equal to 1. Further, we can obtain the following theorem:

\begin{lemma}
\cite{spong2020robot} 
\label{lemma:df_gaussian}
The differential entropy of the Gaussian random variable $X\sim \mathcal{N}(0,\sigma^2)$ is
\begin{equation}
h(X)=\log (\sqrt{2\pi e }\sigma) \notag
\end{equation}
\end{lemma}

\begin{theorem}\label{thm:x_bound}
	If $X^{(n)}$ is $n$-variate Gaussian distributed as $X^{(n)}\sim\mathcal{N}(\mathbf{0},\sigma^2\mathbf{I}_n)$, $\forall \epsilon>0$, $\forall \delta>0$, when $n$ is sufficient large, with probability at least $1-\delta$, a sample $x^{(n)}$ drawn from $X^{(n)}$ has bound
	\begin{align}
    n\sigma^2(1-2\epsilon) < \|x^{(n)}\|_2^2 < n\sigma^2(1+2\epsilon)
    \notag
	\end{align}
\end{theorem}
\begin{proof}
	By Theorem \ref{thm:ts}, the probability that $x^{(n)}$ belongs to the typical set $A_\epsilon^{(n)}(X)$ tending to $1$ as $n\rightarrow \infty$. Thus, $\forall \epsilon>0$, $\forall \delta>0$, $\exists n_0$ s.t. $\forall n \geq n_0$ there exists
	\begin{align}
		\mathbb{P}\left\{ \left|-\frac{1}{n}\log f(x^{(n)}) - h(X) \right| < \epsilon \right\} > 1- \delta\notag
	\end{align}
	Thus, when $n$ is sufficient large with probability at least $1-\delta$, we have 
	\begin{align}
		\left|-\frac{1}{n}\log f(x^{(n)}) - h(X) \right| < \epsilon \notag
	\end{align}
	Since $X^{(n)}\sim\mathcal{N}(\mathbf{0},\sigma^2\mathbf{I}_n)$,
	\begin{align}
		\left|-\frac{1}{n}\log (A e^{-\frac{\|x^{(n)}\|_2^2}{2\sigma^2}}) - h(X) \right| < \epsilon \notag
	\end{align}
	where $A = (2\pi\sigma^2)^{-\frac{n}{2}}$. Thus,
	\begin{align}
		2\sigma^2(\log A+ n(h(X)- \epsilon)) < \|x^{(n)}\|_2^2 < 2\sigma^2(\log A+ n(h(X)+\epsilon)) \notag
	\end{align}
	Since $X\sim\mathcal{N}(0,\sigma^2)$, by Lemma \ref{lemma:df_gaussian}, the differential entropy $h(X)=\log (\sqrt{2\pi e }\sigma)$. Then, 
	\begin{align}
		n\sigma^2(1-2\epsilon) < \|x^{(n)}\|_2^2 < n\sigma^2(1+2\epsilon) \notag
	\end{align}
\end{proof}

By Theorem \ref{thm:x_bound}
, we observe that the left and right bounds of $\|x^{(n)}\|$ are very close to each other. As $n\rightarrow \infty$, $\|x^{(n)}\|_2^2$ approaches $n\sigma^2$. In other words, the artificial Gaussian noise generated during the training process has an almost deterministic $L_2$ norm due to the high dimensionality. The artificial Gaussian noises are sampled on the sphere with radius $\sqrt{n\sigma^2}$ in the $\mathbb{R}^n$ space, not inside or outside, but specifically on the sphere's shell. Consequently, the deep model only learns from samples on the sphere shell. Perturbations that deviate the $L_2$ norms of samples are never encountered by the model, making them difficult to generalize. Thus, the model fail on these samples with small perturbations. 

In the following section, we will derive the typical set of the perturbed Gaussian noise. We will show that the typical set of the perturbed Gaussian noise is a bit larger than the typical set of the Gaussian noise. Moreover, when the perturbation is parallel to the Gaussian noise, it deviates from the typical set of the Gaussian noise the most, which is coincide with the \textit{Observation 2} (parallel to Gaussian).

\subsubsection{Typical Set of the Perturbed Gaussian Noise} 

Next, we explore the effects of adversarial perturbations on the  typical set. If the perturbations are bounded by their $L_2$ norm, the perturbed Gaussian noises belong
to a slightly larger typical set with probability approximately equal to 1.

\begin{lemma}\label{lemma:pdf_diff}
	Let $x^{(n)}$ be a sample drawn from $X^{(n)}\sim\mathcal{N}(\mathbf{0},\sigma^2\mathbf{I}_n)$. A perturbation $\xi\in\mathbb{R}^n$, $\|\xi\|_2\leq \eta$ is added to $x^{(n)}$. When $n$ is sufficient large, $\forall \epsilon>0$, $\log f(x^{(n)}+\xi)$ is bounded around $\log f(x^{(n)})$ as
	\begin{align}
		-\frac{\eta^2 + 2 \eta\sqrt{n\sigma^2(1+2\epsilon)}}{2\sigma^2}
		< \log f(x^{(n)}+\xi) - \log f(x^{(n)})< 
		-\frac{\eta^2 - 2 \eta\sqrt{n\sigma^2(1+2\epsilon)}}{2\sigma^2} \notag
	\end{align}
\end{lemma}
\begin{proof}
	The log-p.d.f. of $x^{(n)}+\xi$ is 
	\begin{align}
		\log f(x^{(n)}+\xi) = \log (A e^{-\frac{\|x^{(n)}+\xi\|_2^2}{2\sigma^2}})\notag
	\end{align}
	where $A = (2\pi\sigma^2)^{-\frac{n}{2}}$. Since
	\begin{align}
		\|x^{(n)}+\xi\|^2_2 &= \|x^{(n)}\|^2_2 + \|\xi\|^2_2 + 2 <\xi, x^{(n)}>  \notag \\
		&\leq \|x^{(n)}\|^2_2 + \|\xi\|^2_2 + 2 \|\xi\|_2\|x^{(n)}\|_2\notag
	\end{align}
    where $<\cdot,\cdot>$ is the inner product.
	The equality holds when $\xi$ has the same direction with $x^{(n)}$. By $\|\xi\|_2\leq \eta$ we have
	\begin{align}
		\|x^{(n)}+\xi\|^2_2 \leq \|x^{(n)}\|^2_2 + \eta^2 + 2 \eta\|x^{(n)}\|_2\notag
	\end{align}
	Thus, the lower bound of the log-p.d.f. of $x^{(n)}+\xi$ is 
	\begin{align}\label{eq:pdf-adv}
		\log f(x^{(n)}+\xi) &= \log (A e^{-\frac{\|x^{(n)}+\xi\|^2_2}{2\sigma^2}})\notag \\
		&\geq \log (A e^{-\frac{\|x^{(n)}\|^2_2 + \eta^2 + 2 \eta\|x^{(n)}\|_2}{2\sigma^2}})\notag \\
		& = \log (A e^{-\frac{\|x^{(n)}\|^2_2}{2\sigma^2}}e^{-\frac{\eta^2 + 2 \eta\|x^{(n)}\|_2}{2\sigma^2}})\notag \\
		&= \log f(x^{(n)}) -\frac{\eta^2 }{2\sigma^2} -\frac{2 \eta\|x^{(n)}\|_2}{2\sigma^2}
	\end{align}
	By Theorem \ref{thm:x_bound}, $\forall \epsilon>0$ when $n$ is sufficient large, we have $\|x^{(n)}\|_2^2 < n\sigma^2(1+2\epsilon)$. Thus,
	\begin{align}
		\log f(x^{(n)}+\xi) - \log f(x^{(n)}) > -\frac{\eta^2 + 2 \eta\sqrt{n\sigma^2(1+2\epsilon)}}{2\sigma^2} \notag
	\end{align}
	With similar proof, we can obtain the upper bound as
	\begin{align}
		\log f(x^{(n)}+\xi) - \log f(x^{(n)}) < -\frac{\eta^2 - 2 \eta\sqrt{n\sigma^2(1+2\epsilon)}}{2\sigma^2} \notag
	\end{align}
\end{proof}

\textit{Remark 1}: when $\xi$ is parallel with $x^{(n)}$, $f(x^{(n)}+\xi)$ and $f(x^{(n)})$ has the maximum difference. It is coincide with the \textit{Observation 2} (parallel to Gaussian) that when the perturbation is parallel with the Gaussian noise, it has the most significant effect degrading the denoising performance..

\begin{theorem}\label{th:B}
	Let $x^{(n)}$ be a sample drawn from $X^{(n)}\sim\mathcal{N}(\mathbf{0},\sigma^2\mathbf{I}_n)$. A perturbation $\xi\in\mathbb{R}^n$, $\|\xi\|_2\leq \eta$ is added to $x^{(n)}$. When $n$ is sufficient large, $\forall \delta>0$, $\forall \epsilon>0$, with probability at least $1-\delta$,  ($x^{(n)}+\xi$) belongs to the typical set $\mathcal{A}_{B_2}^{(n)}$, where
	\begin{align}
		B_2=\frac{\eta^2 + 2 \eta\sqrt{n\sigma^2(1+2\epsilon)}}{2n\sigma^2}+\epsilon \notag
	\end{align}
\end{theorem}
\begin{proof}
	By Lemma \ref{lemma:pdf_diff}, we have
	\begin{align}
		-\frac{1}{n}\log f(x^{(n)}+\xi) - h(X) &< -\frac{1}{n}\log f(x^{(n)}) - h(X) + \frac{\eta^2 + 2 \eta\sqrt{n\sigma^2(1+2\epsilon)}}{2n\sigma^2} \notag \\
		-\frac{1}{n}\log f(x^{(n)}+\xi) - h(X) &> -\frac{1}{n}\log f(x^{(n)}) - h(X) + \frac{\eta^2 - 2 \eta\sqrt{n\sigma^2(1+2\epsilon)}}{2n\sigma^2} \notag
	\end{align}
	According to Theorem \ref{thm:ts}, $\forall \epsilon>0$, $\forall \delta>0$, $\exists n_0$ s.t. $\forall n \geq n_0$ there exists
	\begin{align}
		\mathbb{P}\left\{ \left|-\frac{1}{n}\log f(x^{(n)}) - h(X) \right| < \epsilon \right\} > 1- \delta\notag
	\end{align}
	Thus,
	\begin{align}
		\mathbb{P}\left\{ \left(-\frac{1}{n}\log f(x^{(n)}+\xi) - h(X) \right) \in \Omega \right\} > 1- \delta\notag
	\end{align}
	where
	\begin{align}
		\Omega = (\frac{\eta^2 - 2 \eta\sqrt{n\sigma^2(1+2\epsilon)}}{2n\sigma^2}-\epsilon ,\frac{\eta^2 + 2 \eta\sqrt{n\sigma^2(1+2\epsilon)}}{2n\sigma^2}+\epsilon) \notag
	\end{align}
	Since
	\begin{align}
		\left|\frac{\eta^2 - 2 \eta\sqrt{n\sigma^2(1+2\epsilon)}}{2n\sigma^2}-\epsilon\right| < \left|\frac{\eta^2 + 2 \eta\sqrt{n\sigma^2(1+2\epsilon)}}{2n\sigma^2}+\epsilon\right| \notag
	\end{align}
	Taking
	\begin{align}
		B_2=\frac{\eta^2 + 2 \eta\sqrt{n\sigma^2(1+2\epsilon)}}{2n\sigma^2}+\epsilon \notag
	\end{align}
	we have
	\begin{align}
		\mathbb{P}\left\{ \left|-\frac{1}{n}\log f(x^{(n)}+\xi) - h(X) \right| <B_2 \right\} > 1- \delta\notag
	\end{align}
	Or equivalently expressed by the typical set
	\begin{align}
		\mathbb{P}\left\{(x^{(n)}+\xi) \in \mathcal{A}_{B_2}^{(n)}\right\} > 1- \delta \notag
	\end{align}
\end{proof}

\textit{Remark 2}: since $n$ is very large,  $\frac{\eta^2 + 2 \eta\sqrt{n\sigma^2(1+2\epsilon)}}{2n\sigma^2}$ is small. Thus, $B_2$ is only larger than $\epsilon$ by a little bit.

Theorem \ref{th:B} proved that the perturbations bounded by $L_2$ norm will deviate from the original typical set of the Gaussian noises. The deviation is most significant when the perturbation has the same direction with the Gaussian noise which coincide with our experimental observations. Further, although the adversarial samples deviate from the original typical set, they still remain in another slightly larger typical set.

The following theorems consider the case that perturbations bounded by $L_\infty$ norm. In Lemma \ref{lemma:bound_inf} we derive the explicit solution of the maximum deviation that the perturbation can derivate from the original sample.  In Lemma \ref{lemma:abs_normal_enp}, \ref{thm:x1_bound}, and \ref{lemma:pdf_diff_infty} we derive the maximum deviation of the perturbation in terms of the log-probability density. In Theorem \ref{th:B_inf}, we prove that the perturbed noise bounded by $L_\infty$ norm belongs to a much larger typical set.

\begin{lemma} \label{lemma:bound_inf}
	The optimal solution of the quadratic programming 
	\begin{align}
		&\max_x \| x+c \|^2_2 \notag \\
		&s.t. \; \|x\|_\infty \leq \eta \notag
	\end{align}
	is reached at $x=\text{sign}(c)\eta$, where $x\in\mathbb{R}^n$ and $c\in\mathbb{R}^n$.
\end{lemma}
\begin{proof}
	The above quadratic programming is convex. Thus, by KKT condition, there exists $\mu_i$ and $\lambda_i$, $i=1,\hdots,n$ such that
	\begin{align}
		\mu_i \geq 0 \notag \\
		\lambda_i \geq 0 \notag \\
		x_i - \eta \leq 0 \notag \\
		-x_i - \eta \leq 0 \notag \\
		\mu_i(x_i-\eta) = 0 \notag \\
		\lambda_i(-x_i-\eta) = 0 \notag \\
		-2c_i - 2x_i + \mu_i - \lambda_i =0 \notag
	\end{align}
	Thus, there exists
	\begin{align}
		&\mu_i=0, \lambda_i >0, c < \eta, x_i = -\eta \notag \\
		&\mu_i>0, \lambda_i =0, c > \eta, x_i = \eta \notag
	\end{align}
	Combining the above solutions, 
    $x_i = \text{sign}(c_i)\eta$
    is the necessary and sufficient condition.
\end{proof}

\begin{lemma}\label{lemma:abs_normal_enp}
	Given $X\sim\mathcal{N}(0,\sigma^2)$, $Y = |X|$ has p.d.f.
$f_Y(x;\sigma^2)=\sqrt{\frac2{\pi\sigma^2}}e^{-\frac{x^2}{2\sigma^2}} $
and mean
$
\mu(Y) = \sigma \sqrt{\frac{2}{\pi}} \notag
$.
\end{lemma}
\begin{proof}
 $Y$ has a folded normal distribution, which probability density to the left of $x = 0$ is folded over by taking the absolute value. Let $x$ be a sample drawn from $Y$, which has the following p.d.f.:
	\begin{align}
		f_Y(x;\mu,\sigma^2)=\frac1{\sqrt{2\pi\sigma^2}}e^{-\frac{(x-\mu)^2}{2\sigma^2}}+\frac1{\sqrt{2\pi\sigma^2}}e^{-\frac{(x+\mu)^2}{2\sigma^2}} \notag
	\end{align}
Since the mean is $\mu=0$, it can be further written as
	\begin{align}
		f_Y(x;\sigma^2)=\sqrt{\frac2{\pi\sigma^2}}e^{-\frac{x^2}{2\sigma^2}} \notag
	\end{align}
The mean of $Y$ is
\begin{align}
\mu(Y) &= \int_0^\infty x f_Y(x;\sigma^2) \text{d} x \notag\\
&=\int_0^\infty \frac{1}{2}\sqrt{\frac2{\pi\sigma^2}}e^{-\frac{x^2}{2\sigma^2}} \text{d} x^2  \notag\\
&=\sigma \sqrt{\frac{2}{\pi}}\notag
\end{align}
\end{proof}

\begin{lemma}\label{thm:x1_bound}
	If $X^{(n)}$ is $n$-variate normally distributed as $X^{(n)}\sim\mathcal{N}(\mathbf{0},\sigma^2\mathbf{I}_n)$, when $n$ is sufficient large, $\forall \epsilon>0$, $\forall \delta>0$, with probability at least $1-\delta$, a sample drawn from $X^{(n)}$ has bound
	\begin{align}
		n\sigma\sqrt{\frac{2}{\pi}}-\epsilon < \|x^{(n)}\|_1 <  n\sigma\sqrt{\frac{2}{\pi}}+\epsilon \notag
	\end{align}
\end{lemma}
\begin{proof}
By the definition of $L_1$ norm, there exists 
	\begin{align}
		\left\|x^{(n)}\right\|_1=\sum\limits_{i=1}^n\left|x_i\right|  \notag
	\end{align}
where $x_i\sim \mathcal{N}(0,\sigma^2)$. The absolute value of $x_i$ can be seen as samples i.i.d. drawn from $Y=|X|$. By Theorem \ref{thm:wlln}, we have
	\begin{align}
		\frac{1}{n}\sum\limits_{i=1}^n\left|x_i\right| \rightarrow \mathbb{E}[Y]  \text{ in probability}\notag
	\end{align} 
By Lemma \ref{lemma:abs_normal_enp},
	\begin{align}
		\sum\limits_{i=1}^n\left|x^{(n)}_i\right| \rightarrow n\sigma\sqrt{\frac{2}{\pi}}   \text{ in probability}\notag
	\end{align}
	Thus, when $n$ is sufficient large with probability at least $1-\delta$, we have 
	\begin{align}
		n\sigma\sqrt{\frac{2}{\pi}}-\epsilon<\|x^{(n)}\|_1 < n\sigma\sqrt{\frac{2}{\pi}}+\epsilon \notag
	\end{align}
\end{proof}

\begin{lemma}\label{lemma:pdf_diff_infty}
	Let $x^{(n)}$ be a sample drawn from $X^{(n)}\sim\mathcal{N}(\mathbf{0},\sigma^2\mathbf{I}_n)$. A perturbation $\xi\in\mathbb{R}^n$, $\|\xi\|_\infty \leq \eta$ is added to $x^{(n)}$. When $n$ is sufficient large, $\forall \epsilon>0$,  $\log f(x^{(n)}+\xi)$ is bounded around $\log f(x^{(n)})$ as
	\begin{align}
		-\frac{n\eta^2 + 2 \eta(n\sigma\sqrt{\frac{2}{\pi}}+\epsilon)}{2\sigma^2}
		< \log f(x^{(n)}+\xi) - \log f(x^{(n)})< 
		-\frac{n\eta^2 - 2 \eta(n\sigma\sqrt{\frac{2}{\pi}}+\epsilon)}{2\sigma^2}   \notag
	\end{align}
\end{lemma}
\begin{proof}
	By Lemma \ref{lemma:bound_inf},
	\begin{align}
		\|x^{(n)}+\xi\|^2_2 &\leq \|x^{(n)}+\text{sign}(x^{(n)})\eta \|^2_2 \notag \\
		&= \|x^{(n)}\|^2_2 + n \eta^2 + 2 \|x^{(n)}\|_1 \eta \notag
	\end{align}
	The log-p.d.f. of $x^{(n)}+\xi$ is 
	\begin{align}
		\log f(x^{(n)}+\xi) = \log (A e^{-\frac{\|x^{(n)}+\xi\|_2^2}{2\sigma^2}})\notag
	\end{align}
	where $A = (2\pi\sigma^2)^{-\frac{n}{2}}$. The lower bound of the log-p.d.f. of $x^{(n)}+\xi$ is 
	\begin{align}
		\log f(x^{(n)}+\xi) &= \log (A e^{-\frac{\|x^{(n)}+\xi\|^2_2}{2\sigma^2}})\notag \\
		&\geq \log (A e^{-\frac{\|x^{(n)}\|^2_2 + n\eta^2 + 2 \eta\|x^{(n)}\|_1}{2\sigma^2}})\notag \\
		& = \log (A e^{-\frac{\|x^{(n)}\|^2_2}{2\sigma^2}}e^{-\frac{n\eta^2 + 2 \eta\|x^{(n)}\|_1}{2\sigma^2}})\notag \\
		&= \log f(x^{(n)}) -\frac{n\eta^2 }{2\sigma^2} -\frac{2 \eta\|x^{(n)}\|_1}{2\sigma^2}\notag
	\end{align}
	By Lemma \ref{thm:x1_bound}, $\forall \epsilon>0$ when $n$ is sufficient large, we have $\|x^{(n)}\|_1 < n\sigma\sqrt{\frac{2}{\pi}}+\epsilon$. Thus,
	\begin{align}
		\log f(x^{(n)}+\xi) - \log f(x^{(n)}) > -\frac{n\eta^2 + 2 \eta(n\sigma\sqrt{\frac{2}{\pi}}+\epsilon)}{2\sigma^2} \notag
	\end{align}
	With similar proof, we can obtain the upper bound as
	\begin{align}
		\log f(x^{(n)}+\xi) - \log f(x^{(n)}) < -\frac{n\eta^2 - 2 \eta(n\sigma\sqrt{\frac{2}{\pi}}+\epsilon)}{2\sigma^2} \notag
	\end{align}
\end{proof}

\begin{theorem}\label{th:B_inf}
	Let $x^{(n)}$ be a sample drawn from $X^{(n)}\sim\mathcal{N}(\mathbf{0},\sigma^2\mathbf{I}_n)$. A perturbation $\xi\in\mathbb{R}^n$, $\|\xi\|_\infty\leq \eta$ is added to $x^{(n)}$. When $n$ is sufficient large, $\forall \delta>0$, $\forall \epsilon>0$,  with probability at least $1-\delta$,  ($x^{(n)}+\xi$) belongs to the typical set $\mathcal{A}_{B_\infty}^{(n)}$, where
	\begin{align}
		B_\infty=\frac{\eta^2}{2\sigma^2}+\sqrt{\frac{2}{\pi}}\frac{\eta}{\sigma}+(1+\frac{1}{2n\sigma^2})\epsilon \notag
	\end{align}
\end{theorem}
\begin{proof}
	By Lemma \ref{lemma:pdf_diff_infty}, we have
	\begin{align}
		-\frac{1}{n}\log f(x^{(n)}+\xi) - h(X) &< -\frac{1}{n}\log f(x^{(n)}) - h(X) + \frac{n\eta^2 + 2 \eta(n\sigma\sqrt{\frac{2}{\pi}}+\epsilon)}{2n\sigma^2} \notag \\
		-\frac{1}{n}\log f(x^{(n)}+\xi) - h(X) &> -\frac{1}{n}\log f(x^{(n)}) - h(X) + \frac{n\eta^2 - 2 \eta(n\sigma\sqrt{\frac{2}{\pi}}+\epsilon)}{2n\sigma^2} \notag
	\end{align}
	According to Theorem \ref{thm:ts}, $\forall \epsilon>0$, $\forall \delta>0$, $\exists n_0$ s.t. $\forall n \geq n_0$ there exists
	\begin{align}
		\mathbb{P}\left\{ \left|-\frac{1}{n}\log f(x^{(n)}) - h(X) \right| < \epsilon \right\} > 1- \delta\notag
	\end{align}
	Thus,
	\begin{align}
		\mathbb{P}\left\{ \left(-\frac{1}{n}\log f(x^{(n)}+\xi) - h(X) \right) \in \Omega \right\} > 1- \delta\notag
	\end{align}
	where
	\begin{align}
		\Omega = (\frac{n\eta^2 - 2 \eta(n\sigma\sqrt{\frac{2}{\pi}}+\epsilon)}{2n\sigma^2}-\epsilon ,\frac{n\eta^2 + 2 \eta(n\sigma\sqrt{\frac{2}{\pi}}+\epsilon)}{2n\sigma^2}+\epsilon) \notag
	\end{align}
	Since
	\begin{align}
		\left|\frac{n\eta^2 - 2 \eta(n\sigma\sqrt{\frac{2}{\pi}}+\epsilon)}{2n\sigma^2}-\epsilon\right| < \left|\frac{n\eta^2 + 2 \eta(n\sigma\sqrt{\frac{2}{\pi}}+\epsilon)}{2n\sigma^2}+\epsilon\right| \notag
	\end{align}
	Taking
	\begin{align}
		B_\infty=\frac{\eta^2}{2\sigma^2}+\sqrt{\frac{2}{\pi}}\frac{\eta}{\sigma}+(1+\frac{1}{2n\sigma^2})\epsilon \notag
	\end{align}
	we have
	\begin{align}
		\mathbb{P}\left\{ \left|-\frac{1}{n}\log f(x^{(n)}+\xi) - h(X) \right| <B_\infty \right\} > 1- \delta\notag
	\end{align}
	Or equivalently expressed by the typical set
	\begin{align}
		\mathbb{P}\left\{(x^{(n)}+\xi) \in \mathcal{A}_{B_\infty}^{(n)}\right\} > 1- \delta \notag
	\end{align}
\end{proof}
Comparing to $\epsilon$ and $B_{2}$, $B_\infty$ is much larger.

Above we explored bounds of the perturbed Gaussian noise using typical sets. Next, we explore the volumes of the typical sets to explicitly calculate the size of the typical sets. 

\subsubsection{Volumes of the Typical Sets}
By Theorem \ref{thm:ts}, the volumes of typical sets are bounded and small. We will calculate the volumes of typical sets for the original artificial Gaussian noise, the perturbed Gaussian noise constrained in $L_2$ norm, and the perturbed Gaussian noise constrained in $L_p$ norm.
\begin{proposition}\label{thm:vol_typical}
	For the i.i.d. random variable sequence $X^{(n)}\sim\mathcal{N}(\mathbf{0},\sigma^2\mathbf{I}_n)$, the volume of its typical set  satisfies
	\begin{align}
		(1-\epsilon)(2^{-\epsilon} \sqrt{2\pi e} \sigma )^n \leq \text{Vol}(A_\epsilon^{(n)}(X)) \leq (2^\epsilon \sqrt{2\pi e} \sigma )^n \notag
	\end{align}
	when $n$ is sufficient large.
\end{proposition}
\begin{proof}
	By Theorem \ref{thm:ts},
	\begin{align}
		\text{Vol}(A_\epsilon^{(n)}(X)) &\leq 2^{n(\log (\sqrt{2\pi e} \sigma)+\epsilon)} \notag \\
		&= (2^\epsilon \sqrt{2\pi e} \sigma )^n\notag
	\end{align}
	When $n$ is sufficient large,
	\begin{align}
		\text{Vol}(A_\epsilon^{(n)}(X)) &\geq (1-\epsilon) 2^{n(\log (\sqrt{2\pi e} \sigma)-\epsilon)} \notag \\
		&= (1-\epsilon)(2^{-\epsilon} \sqrt{2\pi e} \sigma )^n \notag
	\end{align}
\end{proof}
Since the bound holds for arbitrary $\epsilon >0$, when $\epsilon$ is small the volume is approximately $\text{Vol}(A_\epsilon^{(n)}(X)) \approx (\sqrt{2\pi e} \sigma )^n$. The volume is approximately equivalent to the volume of a hyper-cubic with edge length of $\sqrt{2\pi e} \sigma$. Compared to the sampling space, the volume of the typical set is finite and small. By Proposition \ref{thm:vol_typical}, $\forall \delta>0$ with probability at least $1-\delta$, we only sample in a tiny subspace of the entire sample space $\mathbb{R}^n$ during the training process of the deep denoising model.

\begin{proposition}\label{thm:vol_typical_adv}
	The volume of the typical set for the perturbed Gaussian noise constrained by the $L_2$ norm has the upper bound
\begin{equation}
\text{Vol}(A_{B_2}^{(n)}(X)) \leq (2^{(\frac{\eta^2+2\eta\sqrt{n\sigma^2(1+2\epsilon)}}{2n\sigma^2}+\epsilon)}\sqrt{2\pi e} \sigma)^n\notag
\end{equation}
\end{proposition}
\begin{proof}
The proof is trivial by Theorem \ref{thm:ts} and Theorem \ref{th:B}.
\end{proof}
When $\epsilon$ is small, the upper bound is approximately equals to
\begin{align}
\text{Vol}(A_{B_2}^{(n)}(X)) &\leq (2^{\frac{\eta^2+2\eta\sqrt{n\sigma^2}}{2n\sigma^2}}\sqrt{2\pi e} \sigma )^n \notag \\
&= (2^{(\frac{\eta^2}{2\sigma^2}n^{-1})}2^{(\frac{\eta}{\sigma}n^{-\frac{1}{2}})}\sqrt{2\pi e} \sigma)^n \notag
\end{align}
which is the volume of a hyper-cubic of length $2^{(\frac{\eta^2}{2\sigma^2}n^{-1})}2^{(\frac{\eta}{\sigma}n^{-\frac{1}{2}})}\sqrt{2\pi e} \sigma$. As we mentioned before, $2^{(\frac{\eta^2}{2\sigma^2}n^{-1})}2^{(\frac{\eta}{\sigma}n^{-\frac{1}{2}})}$ decreases to $1$ exponentially as $n$ increases. Thus, the volume of the typical set for the perturbed Gaussian noise constrained by the $L_2$ norm is slightly larger than the volume of the typical set for the original artificial Gaussian noise.

\begin{proposition}\label{thm:vol_typical_adv_inf}
	The volume of the typical set for the perturbed Gaussian noise constrained by the $L_\infty$ norm has the upper bound
\begin{equation}
\text{Vol}(A_{B_\infty}^{(n)}(X)) \leq (2^{(\frac{\eta^2}{2\sigma^2}+\sqrt{\frac{2}{\pi}}\frac{\eta}{\sigma}+\frac{2n\sigma^2+1}{2n\sigma^2}\epsilon)}\sqrt{2\pi e} \sigma)^n\notag
\end{equation}
\end{proposition}
\begin{proof}
The proof is trivial by Theorem \ref{thm:ts} and Theorem \ref{th:B_inf}.
\end{proof}
When $\epsilon$ is small, the upper bound is approximately equals to
\begin{align}
\text{Vol}(A_{B_\infty}^{(n)}(X)) &\leq (2^{(\frac{\eta^2}{2\sigma^2}+\sqrt{\frac{2}{\pi}}\frac{\eta}{\sigma})}\sqrt{2\pi e} \sigma)^n\notag\\
&= (2^{\frac{\eta^2}{2\sigma^2}}2^{\sqrt{\frac{2}{\pi}}\frac{\eta}{\sigma}}\sqrt{2\pi e} \sigma)^n\notag
\end{align}
which is the volume of a hyper-cubic of length $2^{\frac{\eta^2}{2\sigma^2}}2^{\sqrt{\frac{2}{\pi}}\frac{\eta}{\sigma}}\sqrt{2\pi e} \sigma$. In fact, to ensure the perturbation is non-perceivable under $L_\infty$ constrain, $\eta \ll \sigma$. Thus,  $2^{\frac{\eta^2}{2\sigma^2}}2^{\sqrt{\frac{2}{\pi}}\frac{\eta}{\sigma}}\approx 1$ but still larger than $2^{(\frac{\eta^2}{2\sigma^2}n^{-1})}2^{(\frac{\eta}{\sigma}n^{-\frac{1}{2}})}$ since $n>1$.

Proposition \ref{thm:vol_typical_adv} and Proposition \ref{thm:vol_typical_adv_inf}  provide bounds on the volumes of the typical sets under $L_2$ and $L_\infty$ constrained perturbations. The volumes have inequality 
\begin{align}
\text{Vol}(A_\epsilon^{(n)}(X)) \leq \text{Vol}(A_{B_2}^{(n)}(X)) \leq \text{Vol}(A_{B_\infty}^{(n)}(X))
\end{align}
and are finite in the $\mathbb{R}^n$ space.

\subsubsection{Summary of the Theoretical Explanation}

\begin{itemize}
\item Due to the high dimensionality, the artificial Gaussian noises sampled during the training process belong to the typical set with the probability approximately equal to 1. Specifically, the typical set is a sphere shell in $\mathbb{R}^n$ with radius $\sqrt{n\sigma^2}$ (Theorem \ref{thm:x_bound}).
\item The volume of the typical set is finite and small which is approximately equals to the volume of a hyper-cubic with edge length of $\sqrt{2\pi e}\sigma$ in $\mathbb{R}^n$ (Proposition \ref{thm:vol_typical}).
\item Since deep image denoising models only learn from the typical set during training, they never encounter perturbations that deviate samples from this set. Thus, the models have difficulties to generalize on the adversarial samples.
\item Specifically, when the perturbation aligns with the direction of the Gaussian noise, it deviates the sample from the typical set the most (Theorem \ref{th:B}). This equality condition coincide with the \textit{Observation 2} (parallel to Gaussian).
\item The model's performance under perturbations is related to the direction of the Gaussian noise rather than the model structure or image content, which explains the model similarity in the neighborhoods of all samples. It is also consistent with the \textit{Observation 1} (connected set) and \textit{Observation 3} (content free).
\item Although the perturbations deviate samples out of the typical set, the perturbed adversarial samples still remain within a slightly larger typical set if the perturbation is constrained by $L_2$ or $L_\infty$ norm (Theorem \ref{th:B}, Proposition \ref{thm:vol_typical_adv}, Theorem \ref{th:B_inf}, and Proposition \ref{thm:vol_typical_adv_inf}).
\item Thus, if we sample from the larger typical set instead of the original typical sets, the model will be capable of defending against the perturbations. Since the volumes of the typical sets are finite, collecting the samples is not impossible (Proposition \ref{thm:vol_typical_adv} and Proposition \ref{thm:vol_typical_adv_inf}).
\end{itemize}

\begin{figure}[t!]
	\centering
	\includegraphics[width=.7\textwidth]{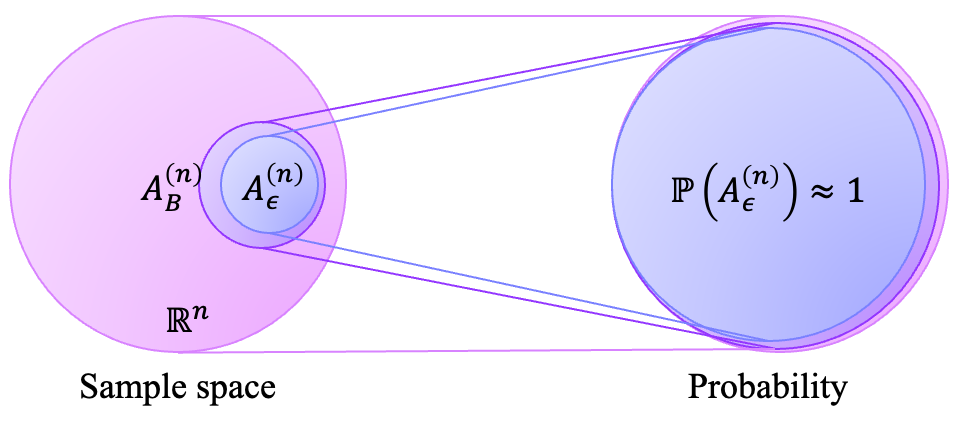}
	\caption{The perturbed samples deviate from the original typical set of the Gaussian noise, but they still remain within a slightly larger typical set. The probabilities of the typical sets tend to one as $n$ increase.}
	\label{typical2}
\end{figure}

\subsection{Adversarial Defense}
As discussed in the previous subsection, the high adversarial transferability is due to the i.i.d. artificial Gaussian noise added during the training process. Due to the curse of dimensionality, the Gaussian noises are sampled from only a small typical set. Adversarial perturbations cause samples to deviate from this typical set, leading to model failure. Since all samples share similar adversarial attack regions, models exhibit high adversarial transferability. Moreover, the sphere-shell topology of the typical set makes the model vulnerable to adversarial attacks. Therefore, to improve the model’s robustness, expanding the sampling region should be a viable approach.

According to Theorem \ref{th:B} and \ref{th:B_inf}, the perturbed samples still remain within a slightly larger typical set. Therefore, augmenting the sampling in this larger typical set, rather than simply sampling within the smaller one, should be beneficial for defending against adversarial attacks. Furthermore, as shown in Theorem \ref{th:B}, perturbations with directions similar to that of the Gaussian noise are adversarial. Hence, the augmented samples should focus on regions aligned with the direction of the Gaussian noise. However, as shown in Figure \ref{typical2}, the probability of the original typical set of the Gaussian noise approaches one as $n$ becomes large. The method of augmenting samples in $A_{B_2}^{(n)}$ or $A_{B_\infty}^{(n)}$, rather than in $A_{\epsilon}^{(n)}$, needs to be further studied.

Typical sets are defined based on the probability densities of the samples. Samples that belong to the larger typical set but not the smaller one are expected to have higher or lower probability densities. But the samples should still remain i.i.d. Gaussian distributed with high probability. When perturbations are in similar directions to the Gaussian noise, the probability densities of the perturbed samples decrease, as shown by Equation \ref{eq:pdf-adv}. Therefore, augmented samples should be collected from i.i.d. Gaussian distributions with lower probability densities. To this end, we propose out-of-distribution typical set sampling (TS) to generate samples within the typical set of the perturbed samples, rather than within the typical set of the Gaussian noise. Models trained using TS sampling can achieve a certain level of resistance to input perturbations.

\subsubsection{TS Sampling} \label{defense}

\begin{algorithm}[t!]\label{alg}
	\caption{TS Sampling}\label{alg:alg1}
	\begin{algorithmic}
		\STATE 
		\STATE {\textsc{INPUT:}}{ Gaussian noise sample $s$, Noise level $\sigma $,  Number of iterations $K$ } 
		\STATE {\textsc{OUTPUT:}}{ Augmented samples $s_{new}$} 
		\STATE
  
		\STATE $f_{new} \gets \text{log\_normal\_pdf}(s,\sigma)$
        \STATE $s_{new} \gets s$
		\FOR {$k = 1:K$} 
    		\STATE {$a \gets \mathcal{N}(\mathbf{0}, \sigma I)$} 
    		\STATE {$f_a \gets \text{log\_normal\_pdf}(a,\sigma)$}
    		\IF {$f_{new} > f_a$}
        		\STATE $s_{new} \gets a$
        		\STATE $f_{new} \gets f_a$
            \ENDIF
        \ENDFOR
        \RETURN $s_{new}$
        \STATE
        \STATE \textbf{function } {$\text{log\_normal\_pdf}(s,\sigma)$:}
        \STATE \hspace{1em} $f\gets \sum_i\log \left(\frac{1}{\sqrt{2\pi}\sigma} \exp{(-\frac{s_i^2}{2\sigma^2})}\right)$
        \STATE \hspace{1em} \textbf{return} {$f$}

	\end{algorithmic}
	\label{alg1}
\end{algorithm}

Based on the above experiments and theories, the effects of adversarial attacks depend solely on the Gaussian noise rather than the image content. Therefore, augmenting only the artificial Gaussian noise is sufficient to improve model robustness. Since both Gaussian noises and perturbed samples can be bound by typical sets, the augmented samples can be generated based on their probability densities. Specifically, the strong adversarial region is characterized by perturbations whose direction is similar to that of Gaussian noise. Consequently, the variance within this region is greater than that of the normally sampled region, resulting in an increase in entropy and a corresponding decrease in log-likelihood. The probability density of this region decreases as shown by Equation \ref{eq:pdf-adv}. In order for the sampling to cover this region, augmented samples should be collected from i.i.d. Gaussian distributions with lower probability densities.

TS sampling operates from the perspective of the probability density distribution, employing screening method to generate Gaussian noise with lower probability density, which is utilized as enhancement samples. As demonstrated in Algorithm 3.1, TS sampling takes the Gaussian noise sample $s\in\mathbb{R}^{H\times W}$ as input, where $H$ is the image height and $W$ is the image width. It then generates new Gaussian noise $a$ with progressively lower probability density to replace $s$. The number of iterations $K$ controls the extent to which the probability density is decreased. TS sampling is a simple algorithm for generating samples with lower probability densities while remaining i.i.d. and following a Gaussian distribution with high probability. The algorithm ensures that, at each iteration, the noise tensor goes towards a lower log-normal distribution probability density. This iterative approach expands the sampling space to the larger typical set after perturbation.

During the training of deep image denoising networks, normal Gaussian noise can be mixed with augmented samples generated through TS sampling to achieve higher robustness. We propose two mixing strategies, differentiated by their sampling ratios: TS-Preserve (TS-Pres.) and TS-Defense (TS-Def.). TS-Pres. aims to maintain or slightly improve the denoising performance of the model while also enhancing its robustness to a certain extent. TS-Def., on the other hand, significantly improves the model's robustness to the point where adversarial samples become nearly ineffective, albeit with a slight reduction in denoising performance. Specifically, TS-Pres. mixes Gaussian noise samples and TS sampling samples at a ratio of 2:1, while TS-Def. mixes them at a ratio of 1:1, utilizing more augmented samples to further enhance the model's robustness. The Gaussian noise samples and TS-Pres. samples are demonstrated in Figure \ref{fenbu}. This figure shows that TS sampling generates samples with lower probability densities. The line in the middle represents the negative value of the differential entropy, indicating the sphere-shell topology to which the samples converge as $n$ tends to infinity.

\begin{figure}[t!]\label{fenbu}
	\centering
	\includegraphics[width=.8\textwidth]{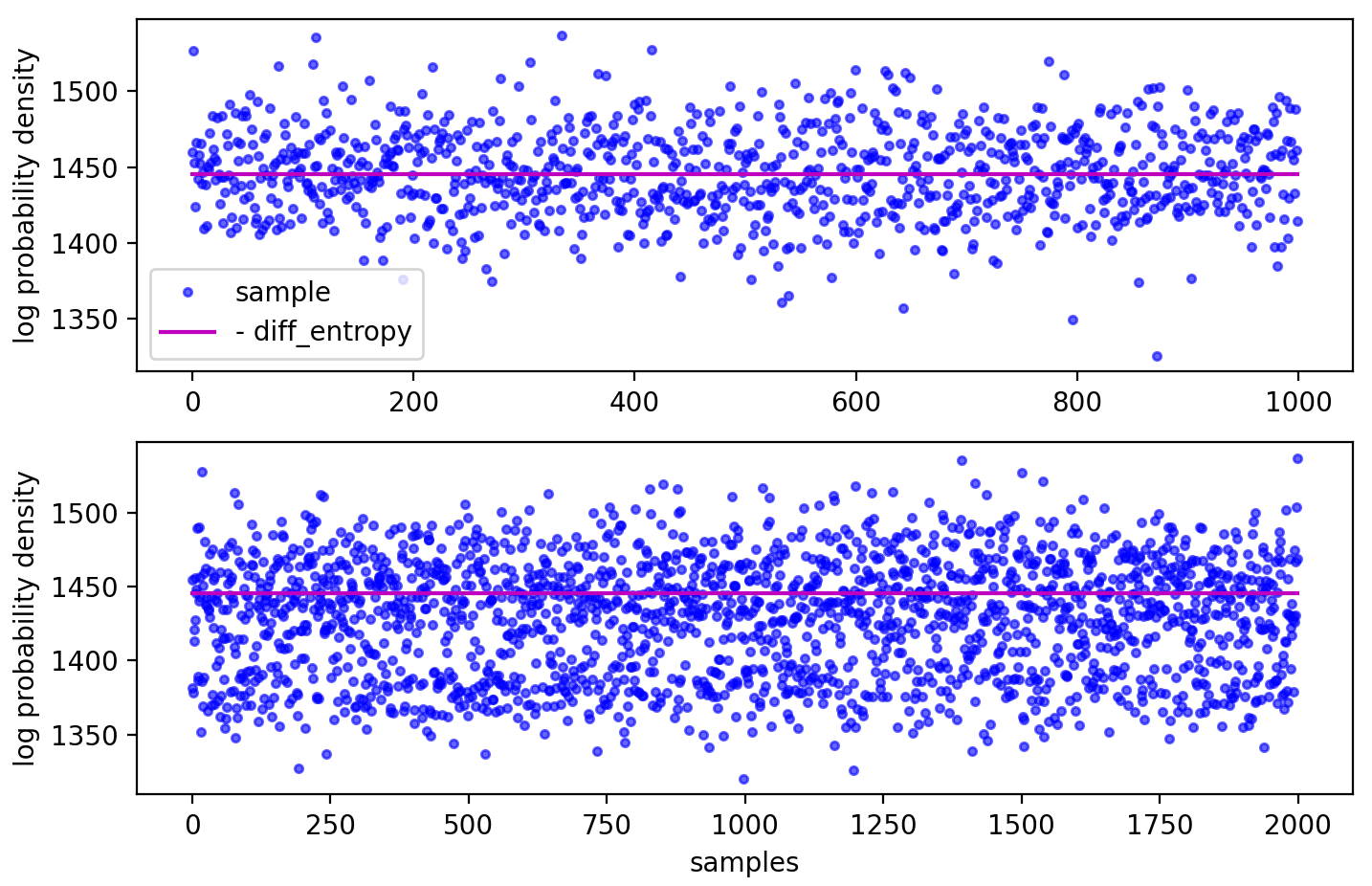}
	\caption{Normal deep image denoising model Gaussian noise sampling schematic (\emph{Top}). TS Gaussian noise sampling schematic (\emph{Down}).}
\end{figure}

\section{Experiments}

\subsection{Experiment Settings}
\subsubsection{Training Details}
During the training process, we followed the DnCNN setup using the CImageNet400 dataset, consisting of 400 grayscale images sized 180 × 180. Subsequently, during testing, we utilized the classical image denoising test dataset Set12 for model testing and comparisons. To test adversarial robustness, we made Set12-Adv-$L_2$ and Set12-Adv-$L_\infty$ datasets on Set12 for DnCNN models using L2-Denoising-PGD and Denoising-PGD \cite{ning2023evaluating}. L2-Denoising-PGD and Denoising-PGD are using the same parameters that the noise level is set to $25/225$, the maximum perturbation $\epsilon$ is set to $3/225$, and the number of iterative steps is set to $5$, and the perturbation step size $\alpha$ is set to $2/225$.

\subsubsection{Evaluation Metrics}
To evaluate model performance, we utilize peak signal-to-noise ratio (PSNR) and structural similarity (SSIM) \cite{ZHOU2019102655}:

\begin{align}
	\operatorname{PSNR}(\mathbf{X}, \mathbf{Y}) & =20 \cdot \log _{10}\left(\frac{\operatorname{MAX}_{\mathrm{I}}}{\operatorname{MSE}(\mathbf{X}, \mathbf{Y})}\right) \\
	\operatorname{SSIM}(\mathbf{X}, \mathbf{Y}) & =\frac{\left(2 \mu_{x} \mu_{y}+c_{1}\right)\left(2 \sigma_{x y}+c_{2}\right)}{\left(\mu_{x}^{2}+\mu_{y}^{2}+c_{1}\right)\left(\sigma_{x}^{2}+\sigma_{y}^{2}+c_{2}\right)}
\end{align}
where $X$ and $Y$ are the two images being used for comparison. $\mu _x$, $\mu _y$, $\sigma _{x}^{2} $, $\sigma _{y}^{2} $ are corresponding mean variance values, and $\sigma _{xy}$ is the covariance. $MAX_I$ is the maximum intensity, which is usually 255 in 8-bit representation.

\subsection{Performance of TS Sampling}

\begin{figure}[t!]\label{show1}
	\centering
	\includegraphics[width=\textwidth]{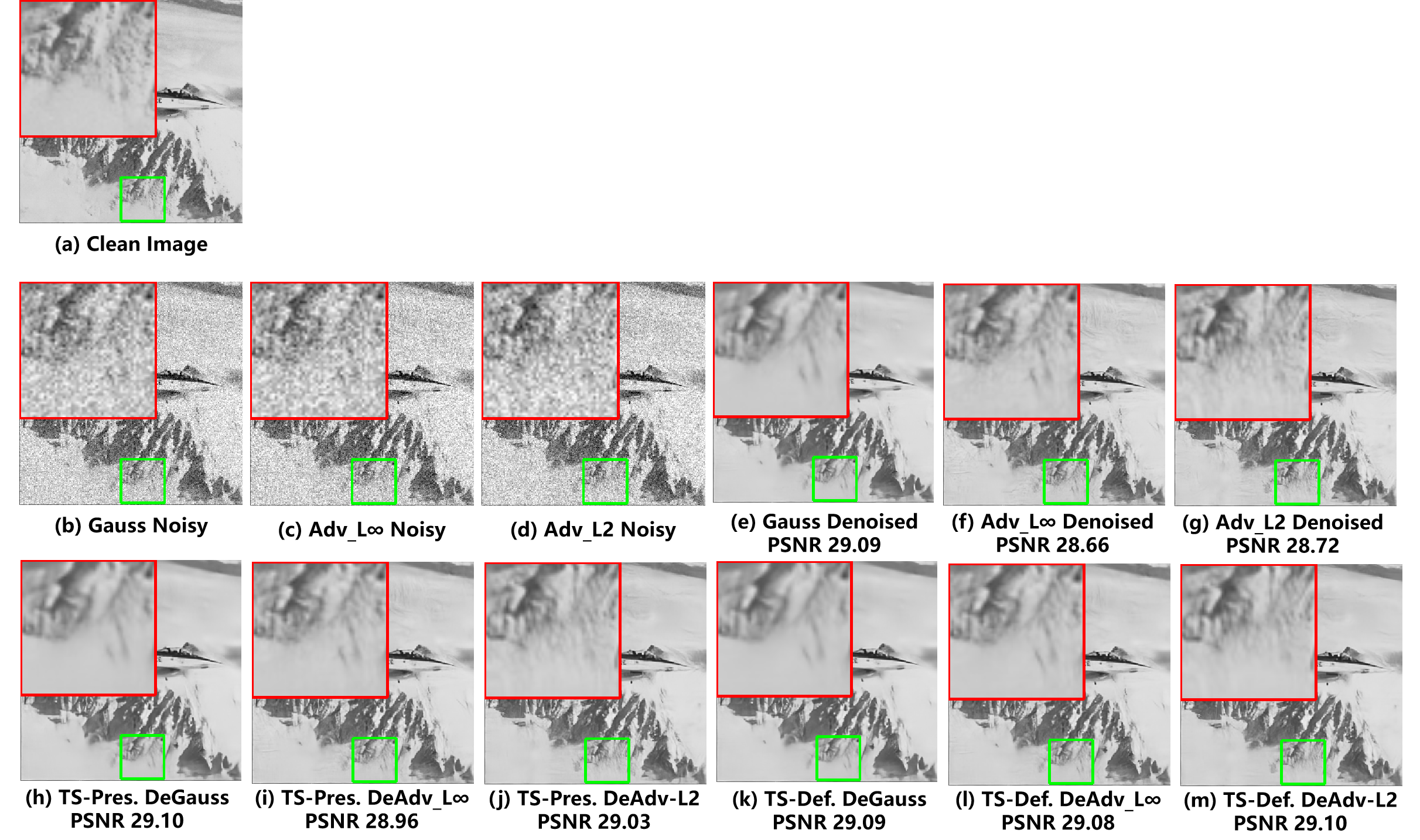}
	\caption{Image denoising results from different training strategies on the 'Airplane' image from the Set12 dataset. The image is corrupted by Gaussian noise, adversarial noise under $L_\infty$ constraint, and adversarial noise under $L2$ constraint, all with a noise level of 25.}
\end{figure}

\begin{figure}[t!]\label{show2}
	\centering
	\includegraphics[width=.7\textwidth]{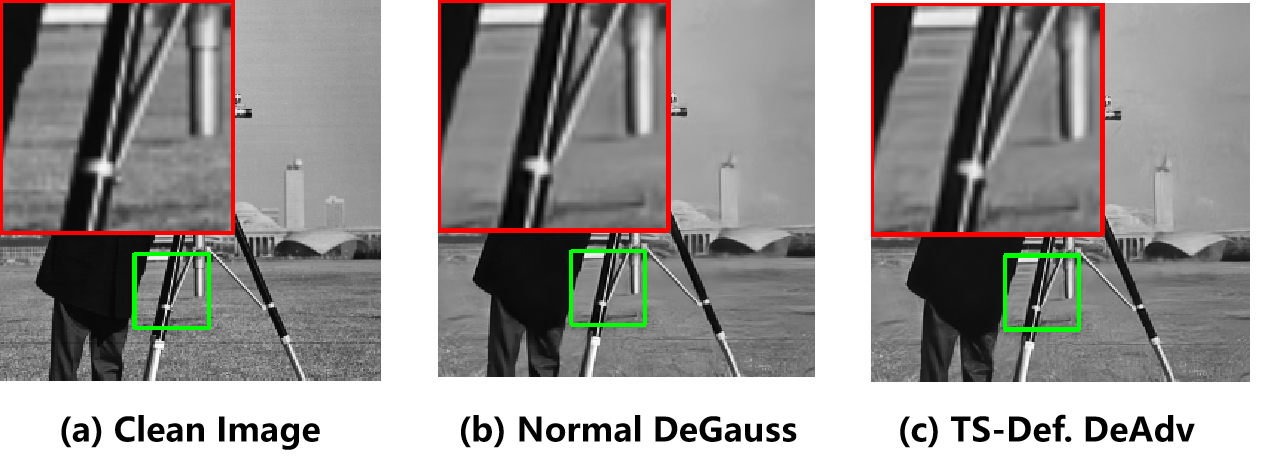}
	\caption{Visual results of adversarial denoising results after TS strategy training and Gaussian denoising results after normal training of DnCNN model.}
\end{figure}

The adversarial defense methods that aim to improve the robustness of the model are often accompanied with the problems of reduced model accuracy and increased computational burden due to the generation of adversarial samples during training. By shifting the distribution of the Gaussian noise samples, TS expands the training data space to accommodate the perturbed Gaussian typical set, thereby reducing the computational load.

To qualitatively evaluate the proposed TS sampling and the data augmentation strategies, we present the denoising results of different training strategies on the classical image 'Airplane' from the Set12 dataset with a noise level of 25, as shown in Figure \ref{show1}. In Figure \ref{show1} (a-c), adversarial perturbations generated under $L_2$ and $L_\infty$ constraints are imperceptible to humans. The denoised results produced by DnCNN are shown in Figure \ref{show1} (e-g). It can be observed that the regularly trained model introduces artifacts in the zoomed-in areas of adversarial samples, and the PSNR value decreases significantly. Figure \ref{show1} (h-m) show the denoised results under TS-Pres. and TS-Def. augmentations. While the model's robustness is improved, its performance under Gaussian noise is well preserved with TS-Def. and even slightly improved with TS-Pres. The robustness is further improved with TS-Def., as the PSNRs are significantly increased in Figure \ref{show1} (l-m), and the artifacts are mitigated.

Figure \ref{show2} further demonstrates that TS sampling is better at processing textured images. It can be observed that, compared to the clean image, the regularly trained DnCNN produces over-smoothed results in the grass area. In contrast, TS-Def. yields a similar denoising result but retains more details in the red box. The details in the red box are closer to those in the original image, which are missing from the Gaussian recovery in regular training. These experiments confirm that the TS training strategies enhance the model's robustness while maintaining or even improving its Gaussian denoising performance.

\begin{sidewaystable}
	\caption{PSNR(dB), SSIM, and MAE for different training strategies on the Set12-Gaussian, Set12-Adv-$L_2$ and, Set12-Adv-$L_\infty$, and  datasets, with a noise level of 25.}
	\fontsize{7}{9}\selectfont
	\centering
	\begin{tabular}{ llllllllllllll|l } 
		\hline
		\hline
		Image &        &Cameraman&House&Pepper&Fishstar&Monarch&Airplane&Parrot&Lena&Barbara&Ship&Man&Couple&Avg.(RES)\\
		\hline
		&&\multicolumn{12}{c}{Set12-Gaussian}&\\
        \hline
        normal &&30.05&33.03&30.81&29.42&30.25&29.13&29.44&32.33&29.88&30.14&30.05&30.03&30.387 \\
        TS-Pres.&PSNR&29.95&33.10&30.85&29.46&30.26&29.10&29.43&32.39&29.93&30.14&30.05&30.03&\textbf{30.390} \\
        TS-Def.&&29.84&33.04&30.79&29.44&30.25&29.08&29.33&32.36&29.90&30.08&30.00&30.00&30.343 \\
		\hline
        normal&&0.87&0.86&0.88&0.87&0.92&0.87&0.87&0.87&0.88&0.81&0.82&0.83&\textbf{0.862} \\
        TS-Pres.&SSIM&0.86&0.86&0.88&0.87&0.92&0.87&0.87&0.87&0.88&0.81&0.82&0.83&0.861 \\
        TS-Def.&&0.86&0.85&0.88&0.87&0.92&0.87&0.86&0.87&0.88&0.81&0.82&0.82&0.859 \\
        \hline
        normal &&5.1&4.0&5.2&6.2&5.3&5.8&5.7&4.4&5.8&5.8&5.8&5.9&5.41\\
        TS-Pres.&MAE&5.1&3.9&5.1&6.2&5.3&5.8&5.7&4.4&5.8&5.8&5.7&5.9&\textbf{5.38}\\
        TS-Def.&&5.1&3.9&5.2&6.2&5.3&5.8&5.8&4.4&5.8&5.8&5.7&5.9&5.40\\
        \hline

        &&\multicolumn{12}{c}{Set12-Adv-$L_2$}&\\
        \hline
        normal&&29.62&31.26&30.04&29.04&29.78&28.72&29.13&30.98&29.17&29.36&29.17&29.20&29.620(-0.767) \\
        TS-Pres.&PSNR&29.81&32.31&30.50&29.31&30.19&29.03&29.26&31.88&29.60&29.86&29.70&29.67&30.087(-0.303) \\
        TS-Def.&&29.77&32.64&30.62&29.39&30.30&29.13&29.23&32.19&29.69&30.01&29.89&29.83&\textbf{30.218(-0.125)} \\
        \hline
        normal&&0.85&0.78&0.84&0.86&0.88&0.84&0.85&0.80&0.85&0.78&0.78&0.79&0.824(-0.038) \\
        TS-Pres.&SSIM&0.86&0.83&0.87&0.87&0.90&0.86&0.86&0.84&0.87&0.80&0.81&0.82&0.850(-0.011) \\
        TS-Def.&&0.86&0.85&0.88&0.87&0.91&0.87&0.86&0.86&0.87&0.81&0.82&0.82&\textbf{0.857(-0.002)} \\
        \hline
		normal&&5.7&5.2&5.9&6.6&5.9&6.4&6.2&5.4&6.5&6.5&6.6&6.7&6.14(+0.73)\\
        TS-Pres.&MAE&5.4&4.5&5.5&6.4&5.5&6.0&6.1&4.7&6.1&6.1&6.1&6.2&5.71(+0.33)\\
        TS-Def.&&5.4&4.2&5.4&6.3&5.4&5.9&6.0&4.5&6.0&5.9&5.9&6.1&\textbf{5.58(+0.18)}\\
        \hline
        &&\multicolumn{12}{c}{Set12-Adv-$L_\infty$}&\\
        \hline
        normal&&29.24&31.58&29.84&28.78&29.60&28.65&28.76&30.96&29.13&29.27&29.18&29.15&29.512(-0.875)\\
        TS-Pres.&PSNR&29.67&32.51&30.34&29.18&29.99&28.96&29.07&31.79&29.56&29.80&29.70&29.66&30.019(-0.371)\\
        TS-Def.&&29.75&32.80&30.51&29.26&30.17&29.08&29.14&32.11&29.67&29.98&29.89&29.83&\textbf{30.183(-0.160)}\\
        \hline
        normal&&0.81&0.80&0.83&0.86&0.88&0.83&0.82&0.80&0.84&0.77&0.78&0.79&0.818(-0.044)\\
        TS-Pres.&SSIM&0.85&0.84&0.86&0.87&0.90&0.86&0.85&0.84&0.86&0.80&0.81&0.82&0.847(-0.014)\\
        TS-Def.&&0.86&0.85&0.87&0.87&0.91&0.87&0.86&0.86&0.87&0.81&0.82&0.83&\textbf{0.856(-0.003)}\\
        \hline
        normal&&6.3&5.0&6.1&6.9&6.0&6.4&6.7&5.4&6.5&6.6&6.6&6.7&6.26(+0.79)\\
        TS-Pres.&MAE&5.7&4.3&5.6&6.5&5.6&6.0&6.2&4.8&6.1&6.1&6.1&6.3&5.79(+0.41)\\
        TS-Def.&&5.6&4.1&5.5&6.4&5.4&5.9&6.2&4.5&6.0&6.0&5.9&6.1&\textbf{5.62(+0.22)}\\

        \hline
        \hline
	\end{tabular}
 \label{ta:proformance}
\end{sidewaystable}

Table \ref{ta:proformance} reports the PSNR (dB), SSIM, and MAE results of DnCNN under different training strategies on the Set12, Set12-Adv-\(L_2\), and Set12-Adv-\(L_\infty\) datasets at a noise level of 25. We note that TS-Pres. enhances model robustness while slightly improving performance on Gaussian noise denoising, making it a universally applicable data augmentation module. TS-Def. significantly reduces the performance gap between adversarial samples and Gaussian noise samples, enabling the model to perform almost equally well on both. The experimental results show that TS sampling can greatly improve the model's robustness by simply changing the noise sampling pattern. In particular, TS-Pres. can maximize model robustness while maintaining or even improving its denoising performance. The success of the TS strategies also suggests that adversarial transferability between models is closely related to the artificial Gaussian noise added during the training process. 

Table \ref{ta:gener} further reports the image denoising results of various denoising models under different training strategies on the Set12, Set12-Adv-$L_2$ and Set12-Adv-$L_\infty$ datasets. The compared methods include DnCNN-B \cite{DNCNN}, CTNet \cite{tian2024cross}, and DeamNet \cite{ren2021adaptive}, where DnCNN-B is a CNN-based blind deep image denoising model, CTNet is a non-blind deep image denoising models with Transformer architecture, and DeamNet is an unfolding-based image denoising model. It can be observed that TS-Pres. produces the best performance in Gaussian denoising for all the models. TS-Def., on the other hand, achieves the best performance in denoising adversarial samples, closely matching the performance on Gaussian denoising. Specifically, TS-Pres. maintains the Gaussian denoising performance of the model while enhancing its robustness. TS-Def. reduces the performance gap between adversarial and Gaussian denoising, although the performance on Gaussian denoising slightly decreases.

\begin{table}
	\caption{Average PSNR(dB)/SSIM results of various denoising models under different training strategies at a noise levels of 25 on the Set12, Set12-Adv-$L_2$, and Set12-Adv-$L_\infty$ datasets.}
	\fontsize{7}{9}\selectfont
	\centering
	\begin{tabular}{lcccccc} 
		\hline
		\hline
		& \multicolumn{2}{c}{Gauss} & \multicolumn{2}{c}{Set12-Adv-$L_2$} & \multicolumn{2}{c}{Set12-Adv-$L_\infty$} \\
		Model & PSNR & SSIM & PSNR & SSIM & PSNR & SSIM \\
		\hline
		&&&normal&&& \\
		\hline
		DnCNN-B   & 30.3350 & 0.8615 & 29.5334 & 0.8221 & 29.3925  & 0.8165 \\
		CTNet          & 30.8325 & 0.8122 & 29.8021 & 0.8201  & 29.7133 & 0.8122  \\
		DeamNet          & 30.6875 & 0.8711  & 29.7558 & 0.8190 & 29.6100 & 0.8094 \\
		\hline
		&&&TS-Pres.&&& \\
		\hline
		DnCNN-B    & \textbf{30.5175}  & \textbf{0.8654}  & 29.8373 & 0.8341  & 29.7083 & 0.8265 \\
		CTNet          & \textbf{30.8675}  & \textbf{0.8728} & 30.1708 & 0.8404  & 30.0525 & 0.8404   \\
		DeamNet          & \textbf{30.6950} & \textbf{0.8711} & 30.4867 & 0.8589 & 30.3950 & 0.8545 \\
            \hline
		&&&TS-Def.&&& \\
		\hline
		DnCNN-B    & 30.2425 & 0.8557 & \textbf{30.0225} & \textbf{0.8476} & \textbf{29.9167} & \textbf{0.8444}  \\
		CTNet          & 30.6425 & 0.8667 & \textbf{30.4750} & \textbf{0.8623} & \textbf{30.4283} & \textbf{0.8619}   \\
		DeamNet          & 30.6483 & 0.8693 & \textbf{30.6042} & \textbf{0.8666} & \textbf{30.5675} & \textbf{0.8664}  \\
		\hline
		\hline
	\end{tabular}
	\label{ta:gener}
\end{table}

In summary, TS sampling has demonstrated a significant advantage in improving model robustness without sacrificing performance on the original task, i.e., Gaussian denoising. Moreover, TS sampling exhibits superior performance in maintaining texture in the Gaussian denoising task.

\subsection{Mixed-noise sampling training method}
Incorporating samples from a neighborhood with a slightly higher noise level around the original sampling point is also an effective method for increasing sampling in regions of lower probability density. Therefore, we propose a mixed-noise sampling training method, which aims to combine data at a specified noise level with data at a marginally higher noise level during the training of deep image denoising models. Figure \ref{fenbu2} illustrates the probability density distributions of mixed noise at noise levels 25 and 26, compared to the probability density distribution obtained using the normal deep image denoising training sampling method. Compared to Figure \ref{fenbu}, the sampling range of the mixed-noise sampling training method has been significantly expanded, accompanied by a substantial downward shift in the log-likelihood density.

\begin{figure}[t!]\label{fenbu2}
	\centering
	\includegraphics[width=.8\textwidth]{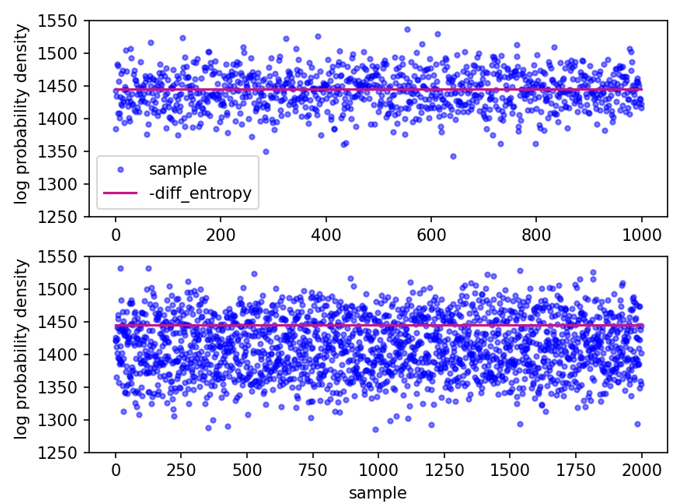}
	\caption{Normal deep image denoising model Gaussian noise sampling schematic (\emph{Top}). Mixed Gaussian noise sampling schematic (\emph{Down}).}
\end{figure}

Table \ref{ta:mixed} presents a comparison of the effects of normal sampling training, TS sampling training, and two types of mixed-noise sampling training: one that combines noise level 25 with noise level 25.5, and another that mixes noise level 25 with noise level 26. In particular, the Gaussian denoising effect of model training for mixed noise levels 25 and 25.5 is adequately preserved, demonstrating a slight improvement in performance; however, the enhancement in model robustness remains insufficient. Conversely, the robustness improvement when training the model with mixed noise levels 25 and 26 is marginally greater; nonetheless, the handling of Gaussian noise tends to deteriorate. It is important to note that TS-Pres. exhibits effects similar to those of the mixed-noise sampling training method and is particularly suitable for scenarios that require stringent preservation of the denoising effectiveness of the model on Gaussian samples. The primary distinction between these two methods is that TS-Pres. converges more rapidly, achieving convergence within 100 epochs, while the mixed-noise sampling training method typically necessitates approximately 200 epochs to reach optimal results. Additionally, TS-Def. allows for a marginal reduction in the Gaussian noise processing effectiveness in exchange for a significant enhancement in robustness, thereby markedly diminishing the impact of adversarial samples.

\begin{table}
	\caption{Average PSNR(dB)/SSIM results of mixed noise and TS sampling training on DnCNN at a noise levels of 25 on the Set12, Set12-Adv-$L_2$, and Set12-Adv-$L_\infty$ datasets.}
	\fontsize{7}{9}\selectfont
	\centering
	\begin{tabular}{lcccccc} 
		\hline
		\hline
		& \multicolumn{2}{c}{Gauss} & \multicolumn{2}{c}{Set12-Adv-$L_2$} & \multicolumn{2}{c}{Set12-Adv-$L_\infty$} \\
		Model & PSNR & SSIM & PSNR & SSIM & PSNR & SSIM \\
		\hline
		Normal   & 30.3871 & 0.8625 & 29.6201 & 0.8238 & 29.5117  & 0.8176 \\
		Mixed w. 25.5          & \textbf{30.4035} & \textbf{0.8625} & 29.9174 & 0.8404  & 29.8235 & 0.8352  \\
		Mixed w. 26          & 30.3839 & 0.8613  & 30.0874 & 0.8499 & 30.0214 & 0.8473 \\
		TS-Pres.          & 30.3900 & 0.8615 & 30.0871 & 0.8504  & 30.0179 & 0.8467  \\
		TS-Def.          & 30.3491 & 0.8589  & \textbf{30.2183} & \textbf{0.8567} & \textbf{30.1881} & \textbf{0.8561} \\
		\hline
		\hline
	\end{tabular}
	\label{ta:mixed}
\end{table}

\section{Discussion and Conclusion}
In this paper, we focus on exploring the reasons why adversarial samples exhibit strong adversarial transferability in the domain of image denoising. Through a series of experiments, we narrowed the reason to the artificial Gaussian noise added to during the training process of the deep models. Additionally, we proposed bounding Gaussian noises and adversarial samples using typical set through the asymptotic equipartition property. By applying the typical sets analysis for image denoising, we proved that adversarial samples deviate slightly from the typical set of the original input distribution, leading to model failures. Inspired by this, we propose a novel adversarial defense method: the Out-of-Distribution Typical Set Sampling (TS) training strategy. TS extends the noise sampling space to the larger typical set of adversarial samples by altering the probability densities of the noise samples. Extensive experimental results demonstrate the effectiveness of the TS strategy for both Gaussian and adversarial denoising. TS can significantly reduce the gap between adversarial denoising and Gaussian denoising, while maintaining or even slightly improving the performance of Gaussian denoising.

There are some insights that were not rigorously proven in this work but are worth mentioning. As demonstrated in our experiments and theories, the model's performance around a sample is related to the direction of the Gaussian noise rather than the image content. This suggests that the deep learning model has effectively learned the distribution of the noise, specifically the Gaussian distribution. Or, for simplicity, deep learning models work as desired. Additionally, while previous works suggest that adversarial attacks are caused by the high dimensionality of the input, their analyses are based on assumptions about the model and the data distribution. This work investigates adversarial attacks in a simpler task: image denoising. In this context, we explicitly identified the topology of the samples and demonstrated how adversarial samples deviate from this topology. Furthermore, we propose that if the distribution of adversarial samples is known, we can sample directly from this distribution to enhance the model robustness instead of using adversarial samples to augment the model.

Future research directions include generalizing the analysis to the other noise distributions, further analysis of the typical set for $L_\infty$-constrained attacks, development of more efficient sampling strategies for the larger typical set, tighter bounds for adversarial samples, application of TS to real-world image denoising, application of TS to plug-and-play inverse problems, and exploration of the safety of deep image denoising methods, among others.

\section*{Acknowledgments}
This work was partially supported by the National Key $R\&D$ Program of China(2023YFC2205900, 2023YFC2205903), the Natural Sciences Foundation of Heilongjiang Province(ZD2022A001), the National Natural Science Foundation of China (12371419, 12171123, U21B2075, 12271130, 12401557), and the Basic Science Center Project of the National Natural Science Foundation of China (62388101).

\bibliographystyle{siamplain}
\bibliography{references}
\end{document}